\documentclass{article}

\usepackage[a4paper, total={6in, 9in}]{geometry}

\usepackage[utf8]{inputenc} 
\usepackage[T1]{fontenc}    
\usepackage{hyperref}       
\usepackage{url}            
\usepackage{booktabs}       
\usepackage{amsfonts}       
\usepackage{nicefrac}       
\usepackage{microtype}      
\usepackage{url}            
\hypersetup{colorlinks, citecolor=blue, linkcolor=red, urlcolor=blue}
\usepackage{mathtools}
\usepackage{thmtools}

\usepackage{algorithm}
\usepackage{algpseudocode}

\usepackage[round]{natbib}

\usepackage{nameref,cleveref}

\def\equationautorefname~#1\null{(#1)\null}

\usepackage{color}
\usepackage{graphicx}
\usepackage{epstopdf}
\usepackage{amsopn}

\usepackage[bb=boondox]{mathalfa}
\usepackage{comment}
\usepackage{defs}
\usepackage{subcaption}
\ifpdf
  \DeclareGraphicsExtensions{.eps,.pdf,.png,.jpg}
\else
  \DeclareGraphicsExtensions{.eps}
\fi

\usepackage{enumitem}
\setlist[enumerate]{leftmargin=.5in}
\setlist[itemize]{leftmargin=.5in}

\newtheorem{definition}{Definition}
\newtheorem{lemma}{Lemma}

\newtheorem{theorem}{Theorem}
\newtheorem{proof}{Proof}

\usepackage{xcolor}

\crefname{lemma}{lemma}{lemmas}
\Crefname{lemma}{Lemma}{Lemmas}

\title{Sparsity-aware generalization theory for deep neural networks}

\author{Ramchandran Muthukumar\thanks{Department of Computer Science \& Mathematical Institute for Data Science, Johns Hopkins University (\texttt{rmuthuk1@jhu.edu}).}
\and Jeremias Sulam\stepcounter{footnote}\stepcounter{footnote}\stepcounter{footnote}\stepcounter{footnote}\stepcounter{footnote}\thanks{Department of Biomedical Engineering  \& Mathematical Institute for Data Science,  Johns Hopkins University (\texttt{jsulam1@jhu.edu}).}
}

\date{}

\makeatletter
\newcommand*{\addFileDependency}[1]{
  \typeout{(#1)}
  \@addtofilelist{#1}
  \IfFileExists{#1}{}{\typeout{No file #1.}}
}
\makeatother

\usepackage{algorithm}
\usepackage{algpseudocode}

\usepackage{sectsty}
\usepackage{titlecaps}
\allsectionsfont{\titlecap}
\usepackage{caption}
\begin{document}

\maketitle

\begin{abstract}
Deep artificial neural networks achieve surprising generalization abilities that remain poorly understood. In this paper, we present a new approach to analyzing generalization for deep feed-forward ReLU networks that takes advantage of the degree of sparsity that is achieved in the hidden layer activations. By developing a framework that accounts for this reduced effective model size for each input sample, we are able to show fundamental trade-offs between sparsity and generalization. Importantly, our results make no strong assumptions about the degree of sparsity achieved by the model, and it improves over recent norm-based approaches. We  illustrate our results numerically, demonstrating non-vacuous bounds when coupled with data-dependent priors in specific settings, even in over-parametrized models.
\end{abstract}

\section{Introduction}
Statistical learning theory seeks to characterize the generalization ability of machine learning models, obtained from finite training data, to unseen test data. The field is by now relatively mature, and several tools exist to provide upper bounds on the generalization error, $R(h)$.
Often the upper bounds depend on the empirical risk, $\hat{R}(h)$, and different characterizations of complexity of the hypothesis class as well as potentially specific data-dependent properties. 
The renewed interest in deep artificial neural network models has demonstrated important limitations of existing tools. 
For example, VC dimension often simply relates to the number of model parameters and is hence insufficient to explain generalization of overparameterized models \citep{Bartlett2019NearlytightVA}. 
Traditional measures based on Rademacher complexity are also often vacuous, as these networks can indeed be trained to fit random noise \citep{zhang2017understanding}. 
Margin bounds have been adapted to deep non-linear networks \citep{Bartlett2017SpectrallynormalizedMB, Golowich2018SizeIndependentSC, Neyshabur2015NormBasedCC, neyshabur2018a}, albeit still unable to provide practically informative results.

An increasing number of studies advocate for non-uniform data-dependent measures to explain generalization in deep learning \citep{nagarajan2019uniform, Prez2020GeneralizationBF, Wei2019DatadependentSC}. 
Of particular interest are those that employ the sensitivity of a data-dependent predictor to parameter perturbations -- sometimes also referred to as \emph{flatness}
\citep{ShaweTaylor1997APA, Neyshabur2017ExploringGI, Dziugaite2017ComputingNG, Arora2018StrongerGB, Li2018OnTG,   nagarajan2018deterministic, Wei2019DatadependentSC, sulam2020adversarial, Banerjee2020DerandomizedPM}.
This observation has received some empirical validation as well \citep{zhang2017understanding, keskar2017on,
izmailov2018averaging, neyshabur2018the, Jiang*2020Fantastic, foret2021sharpnessaware}. 
Among the theoretical results of this line of work, \cite{Arora2018StrongerGB} study the generalization properties of a \emph{compressed} network, and \cite{Dziugaite2017ComputingNG, Neyshabur2017ExploringGI} study a stochastic perturbed version of the original network. The work in \citep{Wei2019DatadependentSC} provides improved bounds on the generalization error of neural networks as measured by a low Jacobian norm with respect to training data, while \cite{Wei2020ImprovedSC} capture the sensitivity of a neural network to perturbations in intermediate layers.
PAC-Bayesian analysis provides an alternate way of studying generalization by incorporating prior knowledge on a distribution of well-performing predictors in a Bayesian setting \citep{McAllester1998SomePT, guedj2019primer, Alquier2021UserfriendlyIT}. 
Recent results \citep{Dziugaite2017ComputingNG, Dziugaite2018EntropySGDOT, Zhou2019NonvacuousGB} have further strengthened the standard PAC-Bayesian analysis by optimizing over the posterior distribution to generate non-vacuous bounds on the expected generalization error of stochastic neural networks. Derandomized versions of PAC-Bayes bounds have also been recently developed \citep{nagarajan2018deterministic, Banerjee2020DerandomizedPM} relying on the sensitivity or \emph{noise resilience} of an obtained predictor. 
All of these works are insightful, alas important gaps remain in understanding generalization in non-linear, over-parameterized networks \citep{Prez2020GeneralizationBF}. 


\textbf{Our contributions.} In this work we employ tools of sensitivity analysis and PAC-Bayes bounds to provide generalization guarantees on deep ReLU feed-forward networks. Our key contribution is to make explicit use of the sparsity achieved by these networks across their different layers, reflecting the fact that only sub-networks, of reduced sizes and complexities, are active at every sample. Similar in spirit to the observations in \cite{https://doi.org/10.48550/arxiv.2202.13216}, we provide conditions under which the set of active neurons (smaller than the number of total neurons) is stable over suitable distributions of networks, with high-probability. In turn, these results allow us to instantiate recent de-randomized PAC-Bayes bounds \citep{nagarajan2018deterministic} and obtain new guarantees that do not depend on the global Lipschitz constant, nor are they exponential in depth. Importantly, our results provide data-dependent non-uniform guarantees that are able to leverage the structure (sparsity) obtained on a specific predictor. 
As we show experimentally, this degree of sparsity -- the reduced number of active neurons -- need not scale linearly with the width of the model or the number of parameters, thus obtaining bounds that are significantly tighter than known results. 
We also illustrate our generalization results on MNIST for models of different width and depth, providing non-vacuous bounds in certain settings.

\textbf{Manuscript organization.} After introducing basic notation, definitions and problem settings, we provide a detailed characterization of stable inactive sets in single-layer feed-forward maps in \Cref{sec:WarmUp}. \Cref{sec:GeneralizationTheory} presents our main results by generalizing our analysis to multiple layers, introducing appropriate distributions over the hypothesis class and tools from de-randomized PAC-Bayes theory. 
We demonstrate our bounds numerically in \Cref{sec:Experiments}, and conclude in \Cref{sec:Conclusion}.

\subsection{Notation and Definitions}\label{subsec: notation} 
Sets and spaces are denoted by capital (and often calligraphic) letters, with the exception of the set $[K] = \{1,\dots, K\}$. 
For a Banach space $\cW$ embedded with norm $\norm{\cdot}_\cW$, 
we denote by $\cB^{\cW}_{r}(\W)$, a bounded ball centered around $\W$ with radius $r$.
Throughout this work, scalar quantities are denoted by lower or upper case (not bold) letters, and vectors with bold lower case letters.
Matrices are denoted by bold upper case letters: $\mt{W}$ is a matrix with \textit{rows} $\vc{w}[i]$. 
We denote by $\mathcal{P}_I$, the index selection operator that restricts input to the coordinates specified in the set $I$. For a vector $\x\in\mathbb{R}^d$ and $I \subset [d]$, $\mathcal{P}_I:\mathbb R^d \to \mathbb R^{|I|}$ is defined as $\mathcal{P}_I(\x) := \x[I]$. For a matrix $\W\in\mathbb{R}^{p\times d}$ and $I\subset [p]$, $\mathcal{P}_I(\W) \in\mathbb{R}^{|I|\times d}$ restricts $\W$ to the \emph{rows} specified by $I$. 
For row and column index sets $I \subset [p]$ and $J \subset [d]$, $\mathcal{P}_{I,J}(\W) \in \mathbb{R}^{|I| \times |J|}$ restricts $\W$ to the corresponding sub-matrix. Throughout this work, we refer to \textit{sparsity} as the \textit{number of zeros} of a vector, so that for $\x\in\mathbb R^d$ with degree of sparsity $s$, $\|\x\|_0=d-s$. 
We denote the induced operator norm by $\norm{\cdot}_2$, 
and the Frobenius norm by $\norm{\cdot}_F$.  In addition, we will often use operator norms of reduced matrices induced by sparsity patterns. To this end, the following definition will be used extensively.

\begin{definition}(Sparse Induced Norms) 
Let $\mt{W} \in \mathbb{R}^{d_2 \times d_1}$ and $(s_2, s_1)$ be sparsity levels such that $0\leq s_1 \leq d_1-1$ and $0\leq s_2\leq d_2-1$. 
We define the $(s_2,s_1)$ sparse induced norm $\|\cdot\|_{(s_2,s_1)}$ as 
\[
\norm{\mt{W}}_{(s_2,s_1)} := 
\max_{\substack{|J_2| = d_2-s_2}}~\;
\max_{\substack{ |J_1| = d_1-s_1}}~\;
\norm{\mathcal{P}_{J_2, J_1}(\mt{W})}_2.
\]
\end{definition}
The sparse induced norm $\norm{\cdot}_{(s_2,s_1)}$ measures the induced operator norm of a worst-case sub-matrix. 
For any two sparsity vectors $(s_2, s_1) \preceq (\hat{s}_2, \hat{s}_1)$, one can show that $\norm{\W}_{(\hat{s}_2, \hat{s}_1)} \leq \norm{\W}_{(s_2, s_1)}$ for any matrix $\W$ (see \Cref{app: lemma: monotone-sparse-norm}). 
In particular, 
\[
\max_{i,j} |\mt{W}[i,j]| =
\norm{\W}_{(d_2-1, d_1-1)} \leq \norm{\W}_{(s_2, s_1)} \leq \norm{\W}_{(0,0)} = \norm{\W}_2.
\]
Thus, the sparse norm interpolates between the maximum absolute entry norm and the operator norm. 
Frequently in our exposition we rely on the case when $s_2 = d_2-1$, thus obtaining 
$\norm{\W}_{(d_2-1, s_1)} = \max_{i\in [d_2]} \max_{|J_1|=d_1-s_1} \norm{\mathcal{P}_{J_1}(\Wvec{}{i})}_2$,
the maximum norm of any reduced row of matrix $\W$.

Outside of the special cases listed above,  
computing the sparse norm for a general $(s_2,s_1)$ has combinatorial complexity. Instead, a modified version of the babel function (see \cite{tropp2003improved}) provides computationally efficient upper bounds\footnote{The particular definition used in this paper is weaker but more computationally efficient than that introduced in \cite{https://doi.org/10.48550/arxiv.2202.13216}.}.
\begin{definition} (Reduced Babel Function \citep{https://doi.org/10.48550/arxiv.2202.13216}) \label[definition]{def: red-babel}
Let $\W \in \mathbb{R}^{d_2 \times d_1}$, the reduced babel function at row sparsity level $s_2 \in \{0, \ldots, d_2-1\}$ and column sparsity level $s_1 \in \{0, \ldots, d_1-1\}$ is defined as\footnote{
When $s_2 = d_2-1, |J_2| = 1$, we simply define $\mu_{(s_2,s_1)}(\W) := 0$.}, 
\[
\mu_{s_2,s_1}(\W) := 
\frac{1}{\norm{\W}^2_{(d_2-1, s_1)}}\;
\underset{\substack{J_2 \subset [d_2],\\ |J_2|=d_2-s_2}}{\max} \;
\max_{j \in J_2} 
\Bigg[
\sum_{\substack{i \in J_2,\\ i \neq j}}\;
\underset{\substack{J_1 \subseteq [d_1]\\ |J_1| = d_1 - s_1}}{\max}
\left\lvert\mathcal{P}_{J_1}(\vc{w}[i]) \mathcal{P}_{J_1}(\vc{w}[j])^T\right\rvert 
\Bigg].
\]
\end{definition}
For the special case when $s_2$ = 0, the reduced babel function is equivalent to the babel function from \cite{tropp2003improved} on the transposed matrix $\W^T$.
We show in \Cref{lemma: bound-submatrix-norm} that the sparse-norm can be bounded using  the reduced babel function and the maximum reduced row norm $\norm{\cdot}_{(d_2-1, s_1)}$, 
\begin{equation}
\label{eq:sparsenormbound}
\norm{\W}_{s_2,s_1} \leq \norm{\W}_{d_2-1,s_1} \sqrt{1+\mu_{s_2, s_1}(\W)}.\end{equation}
See \Cref{app: sparse-norm} for a computationally efficient implementation of the reduced babel function. 

\subsection{Learning Theoretic Framework}
\label{subsec: learning-theoretic-framework}
We consider the task of multi-class classification 
with a bounded input space $\cX = \{\x \in \mathbb{R}^{d_0} ~|~ \norm{\x}_2 \leq \mathsf{M}_\cX\}$ 
and labels $\cY = \{1,\ldots, C\}$ 
from an unknown distribution $\cD_\cZ$ over $\cZ := (\cX \times \cY)$.
We search for a hypothesis in $\mathcal{H} \subset \{h : \cX \rightarrow \cYd \}$ 
that is an accurate predictor of label $y$ given input $\x$. Note that $\cY$ and $\cYd$ need not be the same. In this work, we consider $\cYd = \mathbb{R}^C$, and consider the predicted label of the hypothesis $h$ as $\hat{y}(\x) :=\argmax_{j} [h(\x)]_j$\footnote{The $\argmax$ here is assumed to break ties deterministically.}. 
The quality of prediction of $h$ at $\z = (\x,y)$ is informed by the margin defined as
$ \rho(h, \z) := \left([h(\x)]_y - \argmax_{j \neq y} [h(\x)]_j\right)$. 
If the margin is positive, 
then the predicted label is correct. For a threshold hyper-parameter $\gamma \geq 0$, we define a $\gamma$-threshold 0/1 loss $\ell_{\gamma}$  based on the margin as $\ell_{\gamma} (h, \z) := \mathbb 1\left\{ \rho(h, \z) < \gamma \right\}$.
Note that $\ell_\gamma$ is a stricter version of the traditional zero-one loss $\ell_0$, since $\ell_0(h, \z) \leq \ell_{\gamma}(h,\z)$ for all $\gamma \geq 0$. 
With these elements, the \textit{population risk} (also referred to as \textit{generalization error}) of a hypothesis $R_{\gamma}$ is the expected loss it incurs on a randomly sampled data point, 
$R_{\gamma}(h) := \expect_{\vc{z} \sim \cD_\cZ} \Big[ \ell_{\gamma} \big(h, \vc{z}\big )\Big].$ 
The goal of supervised learning is to obtain a hypothesis with low population risk $R_0(h)$, the probability of misclassification.  
While the true distribution $\cD_\cZ$ is unknown, we assume access to an i.i.d training set $\samp_T =  \{\vc{z}^{(i)},\ldots,\vc{z}^{(m)} \} \sim (\cD_\cZ)^m$  
and we seek to minimize the \textit{empirical risk} $\hat{R}_{\gamma}$, 
the average loss incurred on the training sample $\samp_T$, i.e. 
$\hat{R}_{\gamma} (h) :=
\frac{1}{m}\sum_{i=1}^m \ell_{\gamma} \left(h, \z^{(i)}
\right)$.
We shall later see that for any predictor, $R_0(h)$ can be upper bounded using the stricter empirical risk $\hat{R}_\gamma(h)$ for an appropriately chosen $\gamma >0$. 

In this work, we study the hypothesis class $\cH$ containing feed-forward neural networks with $K$ hidden layers. 
Each hypothesis $h\in \cH$ is identified with its weights $\{\W_k\}_{k=1}^{K+1}$,  and is a sequence of $K$ linear maps $\W_k \in \mathbb{R}^{d_k \times d_{k-1}}$
composed with a nonlinear activation function $\sigma(\cdot)$ and a final linear map $\W_{K+1} \in \mathbb{R}^{C \times d_K}$, 
\[
h(\x_0) 
:= \W_{K+1}\act{\W_k \act{\W_{K-1} \cdots \act{\W_1\x_0} \cdots } }.
\]
We exclude bias from our definitions of feed-forward layers for simplicity\footnote{This is a standard choice in related works, e.g. \cite{Bartlett2017SpectrallynormalizedMB}. Our analysis can be expanded to account for bias.}.  
We denote by $\x_{k}$ the $k^{th}$ hidden layer representation of network $h$ at input $\x_{0}$, so that $\x_{k} := \act{\W_k \x_{k-1}} ~ \forall 1\leq k\leq K,$ and $h(\x) := \W_{K+1}\x_K$.
Throughout this work, the activation function is assumed to be the Rectifying Linear Unit, or ReLU, defined by  $\sigma(x) = \max\{x,0\}$, acting entrywise on an input vector.

\section{Warm up: sparsity in feed-forward maps}\label{sec:WarmUp}
As a precursor to our sensitivity analysis for multi-layer feed-forward networks, we first consider a generic feed-forward map $\Phi(\x):= \sigma(\W \x)$. A na\"ive bound on the norm of the function output is $\norm{\Phi(\x)}_2 \leq \norm{\W}_2 \norm{\x}_2$, but this ignores the sparsity of the output of the feed-forward map (due to the $\relu$). Suppose there exists a set $I$ of inactive indices such that $\mathcal{P}_I(\Phi(\x)) = \mathbf{0}$, i.e. for all $i\in I$, $\vc{w}[i]\cdot \x \leq 0$. In the presence of such an index set, clearly $\norm{\Phi(\x)}_2 \leq \norm{\mathcal{P}_{I^c}(\W)}_2 \norm{\x}_2$\footnote{$I^c$ is the complement of the index set $I$, also referred to as $J$ when clear from context.}.
Thus, estimates of the effective size of the feed-forward output, and other notions such as sensitivity to parameter perturbations, can be refined by accounting for the sparsity of activation patterns. Note that the inactive index set $I$ varies with each input, $\x$, and with the parameters of predictor, $\W$.

For some $\zeta_0, \xi_1, \eta_1 > 0$ and sparsity levels $s_1, s_0$, let $\cX_0 = \{\vc{x} \in \mathbb{R}^{d_0} ~|~ \norm{\vc{x}}_2 \leq \zeta_0, \; \norm{\x}_0 \leq d_0 - s_0
\}$ denote a bounded sparse input domain and
let $\cW_1 := \{\mt{W} \in \mathbb{R}^{d_1 \times d_0} ~|~ \norm{\mt{W}}_{(d_1-1,s_0)} \leq \xi_1,\; \mu_{s_1,s_0}(\mt{W}) \leq \eta_1\}$ denote a parameter space. 
We now define a radius function that measures the amount of relative perturbation 
within which a certain inactive index set is stable. 
\begin{definition} (Sparse local radius\label{def: slr}\footnote{The definition here is inspired by \cite{https://doi.org/10.48550/arxiv.2202.13216} but stronger.})
For any weight $\W \in \mathbb{R}^{d_1\times d_0}$, input $\x \in \mathbb{R}^{d_0}$ and sparsity level $1\leq s_1 \leq d_1$, we define a sparse local radius and a sparse local index set as
\begin{equation}\label{eq: normalized-q}
\rsl(\W, \x, s_1) := \act{
\textsc{sort}\left(-\frac{  \vc{W} \cdot \x}{\xi_1  \zeta_0},\; s_1\right) }, \quad 
I(\W, \x, s_1) := \textsc{Top-k}\left(-\frac{  \vc{W} \cdot \x}{\xi_1 \zeta_0}, s_1\right).
\end{equation}
Here, 
$\textsc{Top-k}(\vc{u},j)$ is the index set of the top $j$ entries in $\vc{u}$, and $\textsc{sort}(\vc{u}, j)$ is its $j^{th}$ largest entry. 
\end{definition} 
We note that when evaluated on a weight $\W \in \cW_1$ and input $\x \in \cX_0$, 
for all sparsity levels the sparse local radius $\rsl(\W, \x, s_1) \in [0,1]$. We denote the sparse local index set as $I$ when clear from the context. We now analyze the stability of the sparse local index set and the resulting 
reduced sensitivity of model output. For brevity, we must defer all proofs to the appendix.

\begin{lemma}\label[lemma]{lemma: slr}
Let $\epsilon_0 \in [0,1]$ be a relative input corruption level and let $\epsilon_1 \in [0,1]$ be the relative weight corruption. 
For the feed-forward map $\Phi$ with weight $\W \in \cW_1$ and input $\x \in \cX_0$, the following statements hold for any output sparsity level $1\leq s_1\leq d_1$, 
\begin{enumerate}
    \item \textbf{Existence of an inactive index set and bounded outputs}: 
    If $\rsl(\W, \x, s_1) > 0$, then the index set $I(\W, \x, s_1)$ is inactive for $\Phi(\x)$. Moreover, $\norm{\Phi(\x)}_2 \leq \xi_1\sqrt{1+\eta_1}\cdot \zeta_0.$
    \item \textbf{Stability of an inactive index set to input and parameter perturbations}:
    Suppose $\xhat$ and $\What$ are perturbed inputs and weights respectively such that, $\norm{\xhat -\x}_0 \leq d_0 - s_0$ and, 
    \begin{align*}
    \text{ }\;
    \frac{\norm{\hat{\x}-\x}_2}{\zeta_0} \leq \epsilon_0 \;\text{ and }\;
    \max\left\{\frac{\norm{\What-\W}_{(d_1-1,s_0)}}{\xi_1} , \frac{\norm{\What-\W}_{(s_1,s_0)}}{\xi_1\sqrt{1+\eta_1}}\right\}\leq \epsilon_1,
    \end{align*}
    and denote $\hat{\Phi}(\x) = \sigma(\hat{\W}\x)$.   If $\rsl(\W, \x, s_1) \geq 
    -1 + (1+\epsilon_0)(1+\epsilon_1)$, then the index set $I(\W, \x, s_1)$ is inactive and stable to perturbations, i.e.\footnote{For notational ease we suppress arguments and let $I = I(\W,\x,s_1)$.} 
    $\mathcal{P}_I(\Phi(\x)) = \mathcal{P}_I(\Phi(\xhat)) = \mathcal{P}_I(\hat{\Phi}(\xhat)) = \mathbf{0}$. Moreover, $\norm{\hat{\Phi}(\xhat) - \Phi(\x)}_2 \leq \left(-1 + (1+\epsilon_0)(1+\epsilon_1)\right)\cdot \xi_1\sqrt{1+\eta_1}\cdot \zeta_0$.
    \item \textbf{Stability of sparse local radius}: 
    For a perturbed input $\xhat$ such that $\norm{\xhat-\x}_0 \leq d_0 - s_0, $ and perturbed weight $\What$, the difference between sparse local radius is bounded  
    \begin{align*}
    \hspace{-30pt}\left\lvert \rsl(\What, \hat{\x},  s_1) - \rsl(\W, \x, s_1) \right \rvert \leq -1 + \left(1+\frac{\norm{\xhat - \x}_2}{\zeta_0}\right)\left(1+\frac{\norm{\What-\W}_{(d_1-1,s_0)}}{\xi_1}\right).
    \end{align*}
\end{enumerate}
\end{lemma}

A key takeaway of this Lemma (see Appendix \ref{app: lemma: slr} for its proof) is that one can obtain tighter bounds, on both the size of the network output as well as its sensitivity to corruptions, if the corresponding sparse local radius is sufficiently large. 
The results above quantify these notions for a given sample. In the next section, we will leverage this characterization within the framework of PAC-Bayes analysis to provide a generalization bound for feed-forward networks.%

\section{A sparsity-aware generalization theory}\label{sec:GeneralizationTheory}
We shall construct non-uniform data-dependent generalization bounds for feed-forward networks based on a local sensitivity analysis of deep ReLU networks, employing the intuition from the previous section. To do so, we will first study the size of the layer outputs using \Cref{def: red-babel}, then  measure the sensitivity in layer outputs to parameter perturbations using \Cref{lemma: slr} across multiple layers, and finally leverage a derandomized PAC-Bayes result from \cite{nagarajan2018deterministic} (see \Cref{app: subsec: derandom}). 
Before embarking on the analysis, we note the following convenient property of the margin for any two predictors $h, \hat{h}$ from \cite[Lemma A.3]{Bartlett2017SpectrallynormalizedMB}, 
\[
\left\lvert \left(h(\x)_y - \max_{j\neq y} h(\x)_j \right)
- \left(\hat{h}(\x)_y - \max_{j\neq y} \hat{h}(\x)_j\right)
\right\rvert \leq 2\norm{\hat{h}(\x) - h(\x)}_\infty.
\]
Hence, quantifying the sensitivity of the predictor outputs will inform the sensitivity of the loss. Similar to other works \citep{nagarajan2018deterministic, Banerjee2020DerandomizedPM}, our generalization bound will be derived by studying the sensitivity of neural networks upon perturbations to the layer weights.

For the entirety of this section, we fix a set of \textit{base hyper-parameters} that determine a specific class of neural networks, 
the variance of a posterior distribution over networks, and the resolution (via a sparsity vector) at which the generalization is measured -- see \Cref{table:basehyp} for reference. 
We denote by $\vc{s} = \{s_1, \ldots, s_K\}$ a vector of layer-wise sparsity levels, which reflects the inductive bias of the learner on the potential degree of sparsity of a trained network on the training data. Next we define two 
hyper-parameters, $\bm{\xi} := \{\xi_1, \ldots, \xi_{K+1}\}$ where $\xi_k > 0$ bounds the sparse norm $\norm{\cdot}_{(d_k-1, s_{k-1})}$ of the layer weights and $\bm{\eta} := \{\eta_1, \ldots, \eta_{K}\}$ where $\eta_k > 0$ bounds the reduced babel function $\mu_{s_k, s_{k-1}}(\cdot)$ of the layer weights. 
Finally, we let $\bm{\epsilon}:= \{\epsilon_1, \ldots, \epsilon_{K+1}\}$ with $\epsilon_k > 0$ bound the amount of relative perturbation in the weights. 
This section treats the quartet $(\vc{s}, \bm{\xi}, \bm{\eta}, \bm{\epsilon})$ as constants\footnote{Unless otherwise specified we let $s_0 = s_{K+1} = 0$ and $\epsilon_0 = 0$.}, while in the next section we shall discuss appropriate values for these hyper-parameters.

\begin{table} 
\centering
\begin{tabular}{|c|c|}
\hline
$\vc{s} = \{s_1, \ldots, s_k\},\;\; 0 \leq s_k \leq d_k -1$ & Layer wise sparsity vector  \\ 
\hline
$\bm{\xi} = \{\xi_1, \ldots, \xi_{K+1}\},\;\; 0 \leq \xi_k $ & Layer wise bound on $\norm{\cdot}_{(d_k-1, s_{k-1})}$ \\
\hline 
$\bm{\eta} = \{\eta_1, \ldots, \eta_{K}\},\;\; 0 \leq \eta_k $ & Layer wise bound on $\mu_{s_k, s_{k-1}}(\cdot)$ \\
\hline
$\bm{\epsilon} = \{\epsilon_1, \ldots, \epsilon_{K+1}\},\;\; 0 \leq \epsilon_k $ & Layer wise bound on relative perturbation\\
 \hline
\end{tabular}
\caption{Independent base hyper-parameters}
\label{table:basehyp}
\end{table}
 \begin{definition}(Norm bounded feed-forward networks)\label{def: bounded-ffn}
We define below the 
parameter domain $\mathcal{W}_{k}$ and a class of feed-forward networks $\cH_{K+1}$ with $K$-hidden layers, 
\begin{align}\label{eq: domain}
\nonumber    
&\mathcal{W}_{k} := \left\{ \mt{W} \in \mathbb{R}^{d_{k} \times d_{k-1}} ~|~ \norm{\mt{W}}_{(d_k -1, s_{k-1})} \leq \xi_k,\quad 
\mu_{s_{k}, s_{k-1}}(\mt{W}) \leq \eta_k,\right\},\; \forall\; k\in [K], \; \\
\nonumber 
&\cH := \left \{ 
h(\cdot) := \W_{K+1} \act{\W_K \cdots \act{\W_1 \cdot}} ~\vert~ \norm{\W_{K+1}}_{(C-1, s_{K})} \leq \xi_{K+1}, \;\mt{W}_k \in \mathcal{W}_{k}, \; \forall\; k\in [K] 
\right\}.
\end{align}
\end{definition}
To 
measure the local sensitivity of the network outputs, it will be useful to formalize a notion of local neighborhood for networks. 
\begin{definition}(Local Neighbourhood)\label{def: nearby-networks} 
Given $h \in \cH$, 
define $\cB(h, \bm{\epsilon})$ to be the local neighbourhood around $h$ containing perturbed networks $\hat{h}$ with weights $\{\What_j\}_{k=1}^{K+1}$ such that at each layer $k$\footnote{For the last layer we only require $\norm{\What_{K+1}-\W_{K+1}}_{C-1,s_K} \leq \epsilon_{K+1} \cdot \xi_{K+1}.$}, 
\begin{align*}
\max\left\{ \frac{\norm{\What_k-\W_k}_{(s_k, s_{k-1})}}{\xi_k\sqrt{1+\eta_k}}, \frac{\norm{\What_k-\W_k}_{(d_k-1, s_{k-1})}}{\xi_k}  \right\} \leq \epsilon_k.
\end{align*}
\end{definition}
It will be useful to understand the probability that $\hat{h} \in \mathcal{B}(h, \bm{\epsilon})$ when the perturbations to each layer weight are random, in particular from Gaussian distributions over feed-forward networks:
\begin{definition}(Entrywise Gaussian)
Let $h \in \cH$ be any network with $K+1$ layers, 
and let $\bm{\sigma}^2 := \{\sigma_1^2, \ldots, \sigma^2_{K+1}\}$ be a layer-wise variance. 
We denote by $\mathcal{N}(h, \bm{\sigma^2})$ a distribution with mean network $h$ such that for any $\hat{h} \sim \mathcal{N}(h, \bm{\sigma}^2)$ with layer weights $\What_{k}$, each entry $\What_k[i,j] \sim \mathcal{N}(\W_k[i,j], \sigma_k^2)$. 
\end{definition}

\subsection{Sensitivity of network output}
Given a predictor $h\in\cH$,
note that the size of a network output for any given input is bounded by $\norm{h(\x_0)}_2 \leq \prod_{k=1}^{K+1} \norm{\W_k}_2 \mathsf{M}_\cX$, which ignores the sparsity of the intermediate layers. 
We will now generalize the result in \Cref{lemma: slr} by making use of the inactive index sets at every layer $I_{k}$, such that $\mathcal{P}_{I_{k}}(\x_k) = \mathbf{0}$, obtaining a tighter (input dependent) characterization of sensitivity to perturbations of the network.
For notational convenience, we define two additional dependent notations: we let $\zeta_0 := \mathsf{M}_\cX$ and $\zeta_k:= \xi_k \sqrt{1+\eta_k} \cdot \zeta_{k-1} = \mathsf{M}_{\cX} \prod_{n=1}^k \xi_n \sqrt{1+\eta_n}$ denote a bound on the layer-wise size of the outputs. At the final layer, we let $\zeta_{K+1}:= \xi_{K+1} \zeta_K$ as a bound on the network output. Additionally, we define $\gamma_k:= -1 + \prod_{n=1}^{k} (1+\epsilon_n)$ as a threshold on the sparse local radius evaluated at each layer -- see \Cref{table:dependent} for a summary. In the last layer, we let this value $\gamma_{K+1}$ represent the desired margin.
For networks $\hat{h}$ with perturbed weights $\What$, we denote by $\xhat_k := \act{\What_k \xhat_{k-1}}$ the perturbed layer representation corresponding to input $\x_0$.
\begin{table}  
\centering
\begin{tabular}{|c|c|}
\hline
$  
\zeta_k := \xi_k \sqrt{1+\eta_k} \cdot \zeta_{k-1}$ 
, \;\; $\forall\; k \in [K]$& 
Bound on norm of layer outputs\\
 \hline
$\zeta_{K+1} := \xi_{K+1}\zeta_K$\;\; & Bound on norm of network output\\
\hline 
$\gamma_k := -1 + \prod_{n=1}^{k} (1 + \epsilon_n)$,\;\; $\forall\; k \in [K+1]$ & Layer wise threshold for local radius\\
 \hline
 $r_k(h,\z) := 
 \act{
\mathrm{sort}\left(-\left[ \frac{  \Wvec{k}{i} \cdot \x_{k-1} }{\xi_k \zeta_{k-1}}  \right]_{i=1}^{d_k},\; d_k -s_k\right) }$ & Layer-wise sparse local radius\\[2pt]
\hline
\end{tabular}
\caption{Layer-wise bounds and thresholds.} 
\label{table:dependent}
\end{table}

\begin{definition}(Layer-wise sparse local radius)\label{def: layer-slr}
Let $h$ be any feed-forward network with weighs $\W_k \in \mathbb{R}^{d_k\times d_{k-1}}$, and let $\x_0 \in \mathbb{R}^{d_0}$.
We define a layer-wise sparse local radius and a layer-wise inactive index set as below,
\begin{equation}\label{eq: layer-slr}
\nonumber I_k(h,\x_0) := 
\textsc{Top-k}\left(- \frac{  \vc{W}_k \cdot \x_{k-1}}{\xi_k \zeta_{k-1}} , s_k\right), \quad
r_k
(h, \x_0) 
\nonumber := 
\act{
\mathrm{sort}\left(-\frac{  \vc{W}_k \cdot \x_{k-1}}{\xi_k \zeta_{k-1}},\; s_k\right) 
}.
\end{equation}
\end{definition} 
Definition \ref{def: layer-slr} now allows us, by employing \Cref{lemma: slr}, to generalize our previous observations to entire network models, as we now show.
\begin{theorem}\label[proposition]{proposition: reduced-size-multi}
Let $h \in \cH$, if at each layer $k$ the layer-wise sparse local radius is nontrivial, i.e. $\forall~k\in [K],~~ r_k(h, \x_0) > 0$.
Then the index sets $I_k(h,\x_0)$ are inactive at layer $k$ and the size of the hidden layer representations and the network output are bounded as follows, 
\begin{align}\label{eq: local-layer-size-3}
\forall~k\in [K],~~ \norm{\x_k}_2 
&\leq 
\zeta_k,
~~\text{ and }~~
\norm{h(\x_0)}_{\infty} 
\leq 
\zeta_{K+1}.
\end{align}
\end{theorem}

\noindent In a similar vein, we can characterize the sensitivity of the network to parameter perturbations. 

\begin{theorem}\label[proposition]{proposition: reduced-sensitivity}
Let $h \in \cH$ and let $\hat{h} \in \mathcal{B}(h,\bm{\epsilon})$ be a nearby perturbed predictor with weights $\{\What_k\}$. 
If each layer-wise sparse local radius is sufficiently large, i.e. $\forall~k\in [K],~r_k(h, \x_0) \geq \gamma_{k}$,
then the index sets $I_k(h,\x_0)$ are inactive for the perturbed layer representations $\xhat_k$ and the distance between the layer representations and the network output are bounded as follows, 
\begin{align}\label{eq: local-layer-size-3}
\forall~k\in [K],~~ \norm{\xhat_k - \x_k}_2 
&\leq 
\zeta_k \cdot \gamma_k,
~~\text{ and }~~
\norm{\hat{h}(\x_0) - h(\x_0)}_{\infty} 
\leq 
\zeta_{K+1} \cdot \gamma_{K+1}.
\end{align}
\end{theorem}
Proofs of the above propositions can be found in \ref{app: proposition: reduced-size-multi} and \ref{app: proposition: reduced-sensitivity} respectively. 
\subsection{Sparsity-aware generalization}
We are now ready to state our main theorem on generalization of feed-forward networks that leverages improved sensitivity of network outputs due to stable inactive index sets. 

\begin{theorem}\label{theorem: sparsity-aware-single}
Let $\mathcal P$ 
be any prior distribution over depth-$(K+1)$ feed-forward network 
chosen independently of the training sample. 
Let $h \in \cH$ be any feed-forward network (possibly trained on sample data), with $\cH$ determined by fixed base hyper-parameters $(\vc{s}, \bm{\epsilon}, \bm{\xi}, \bm{\eta})$, and denote the sparse loss by $\ell_{\mathrm{sparse}}(h,\x) = \mathbb 1\{ \exists  \; k,\;  r_{k}(h, \x) < 3 \gamma_k \}$.  With probability at least $(1-\delta)$ over the choice of i.i.d training sample $\samp_T$ of size $m$, 
the generalization error of $h$ is bounded as follows,
\begin{equation*}
R_{0}(h) \leq  \hat{R}_{4\zeta_{K+1}\gamma_{K+1}}(h)
 + \frac{2K}{m}\sum_{\x^{(i)} \in \samp_T} 
\ell_{\mathrm{sparse}}(h,\x^{(i)})
 + \tilde{\mathcal O} \left(\sqrt{\frac{\mathrm{KL}\left(  
  \mathcal{N}\left(h, \bm{\sigma}^2_{\mathrm{sparse}}\right)
 \;|| \;
 \mathcal{P}
 \right)
}{m}}\right)
\end{equation*}
where $\bm{\sigma}_{\mathrm{sparse}} = \{\sigma_{1}, \ldots, \sigma_{K}\}$ 
is defined by $\sigma_{k}:= \epsilon_k \cdot \frac{\xi_k }{4 \sqrt{ 2d_\mathrm{eff} + \log\left(2(K+1)\sqrt{m}\right)}}$, and where  
$d_{\mathrm{eff}} := \max_{k \in [K]} \frac{\left(d_k - s_k) \log(d_k) + (d_{k-1} -s_{k-1})\log(d_{k-1}\right)}{2}$ is an effective layer width\footnote{
We note the effective width is at worst $\max_k d_k \log(d_k)$ and could be larger than actual width depending on the sparsity vector $\vc{s}$. In contrast, for large $\vc{s}$, $d_{\mathrm{eff}} \ll \max_k d_k$.}.
\end{theorem}

The notation $\tilde{\mathcal O}$ above hides logarithmic factors (see \Cref{app: sparsity-aware} for a complete version of the bound). This result bounds the generalization error of a trained predictor as a function of three terms. Besides the empirical risk with margin threshold $4\zeta_{K+1}\gamma_{K+1}$, 
the risk is upper bounded by an empirical sparse loss that measures the proportion of samples (in the training data) that do not achieve a sufficiently large sparse radius at any layer. Lastly, as is characteristic in PAC-Bayes bounds, we see a term that depends on the distance between the prior and posterior distributions, the latter centered at the obtained (data-dependent) predictor. 
The posterior variance $\bm{\sigma^2}_{\mathrm{sparse}}$ is determined entirely by the base hyper-parameters. 
Finally, note that the result above holds for any prior distribution $\mathcal{P}$. Before moving on, we comment on the specific factors influencing this bound.

\paragraph{Sparsity.} The result above depends on the sparsity by the choice of the parameter $\vc s$. One can always instantiate the above result for $\vc{s}=\vc{0}$, corresponding to a global sensitivity analysis. At this trivial choice, the sparsity loss vanishes (because the sparse radius is infinite) and the bound is equivalent to an improved (derandomized) version of the results by \cite{neyshabur2018a}. 
The formulation in \Cref{theorem: sparsity-aware-single}  enables a continuum of choices (via hyper-parameters) suited to the trained predictor and sample data. 
A larger degree of sparsity at every layer results in a tighter bound since the upper bounds to the sensitivity of the predictor is reduced (as only reduced matrices are involved in its computation). In turn, this reduced sensitivity leads to a lower empirical margin risk by way of a lower threshold $4 \zeta_{K+1} \gamma_{K+1}$. 
Furthermore, the effective width -- determining the scale of posterior -- is at worst $\max_k d_k \log(d_k)$ (for $\vc{s}=0$) , but for large $\vc{s}$, $d_{\mathrm{eff}} \ll \max_k d_k$.

\paragraph{Sensitivity.}
Standard sensitivity-based generalization bounds generally  depend directly on the global Lipschitz constant that scales as $\mathcal{O}(\prod_{k=1}^{K} \norm{\W_k}_2)$. For even moderate-size models, such dependence can render the bounds vacuous. 
Further recent studies suggest that the layer norms 
can even increase with the size of the training sets showing that, even for under-parameterized models, generalization bounds may be vacuous \citep{nagarajan2019uniform}. 
Our generalization bound does \emph{not} scale with the reduced Lipschitz constant $\zeta_{K+1}$: while larger (reduced) Lipschitz constants can render the empirical sparse loss closer to its maximum value of 1, the bound remains controlled due to our choice of modelling \emph{relative} perturbations of model parameters.

\paragraph{Dependence on Depth.}
Unlike recent results \citep{Bartlett2017SpectrallynormalizedMB, Neyshabur2015NormBasedCC, neyshabur2018a, neyshabur2018the}, our bound is not  exponential with depth. However, the sensitivity bounds $\zeta_k$ and radius thresholds $\gamma_k$ are themselves exponential in depth. While the empirical risk and sparse loss terms in the generalization bounds depend on $\zeta_k, \gamma_k$, they are bounded in $[0,1]$. In turn, by choosing the prior to be a Gaussian $P=\mathcal{N}(h_{\mathrm{prior}}, \bm{\sigma}_{\mathrm{sparse}}^2)$, the KL-divergence term can be decomposed into layer-wise contributions, 
$\mathrm{KL}\left(  
  \mathcal{N}\left(h, \bm{\sigma}_{\mathrm{sparse}}^2\right)
 \;|| \;
 \mathcal{N}(h_{\mathrm{prior}}, \bm{\sigma}_{\mathrm{sparse}}^2) 
 \right)
= \sum_{k=1}^{K+1} \frac{\norm{\W_k - \W_{\mathrm{prior},k}}^2_F}{2 \sigma^{2}_{k}}. 
$
Hence, the KL divergence term does not scale with the product of the relative perturbations (like $\gamma_k$) or the product of layer norms (like $\zeta_k$).

\paragraph{Comparison to related work.} 
Besides the relation to some of the works that have been mentioned previously, our contribution is most closely related to those approaches that employ different notions of reduced effective models in developing generalization bounds. \cite{Arora2018StrongerGB} do this via a \emph{compression} argument, alas the resulting bound holds for the compressed network and not the original one. \cite{Neyshabur2017ExploringGI} develops PAC-Bayes bounds that clearly reflect the importance of \textit{flatness}, which in our terms refers to the loss effective sensitivity of the obtained predictor. Similar in spirit to our results, \cite{nagarajan2018deterministic} capture a notion of reduced active size of the model and presenting their derandomized PAC-Bayes bound (which we centrally employ here). While avoiding exponential dependence on depth, their result depends inversely with the minimum absolute pre-activation level at each layer, which can be arbitrarily small (and thus, the bound becomes arbitrarily large). Our analysis, as represented by \Cref{lemma: slr}, circumvents this limitation. Our constructions on normalized sparse radius have close connections with the \textit{normalized margins} from \cite{Wei2020ImprovedSC}, and our use of augmented loss function (such as our \textit{sparse loss}) resemble the ones proposed in \cite{Wei2019DatadependentSC}. Most recently, \cite{galanti2023norm} analyze the complexity of compositionally sparse networks, however the sparsity stems from the convolutional nature of the filters rather than as a data-dependent (and sample dependent) property.

\subsection{Hyper-parameter search} \label{sec:HyperParameterSearch}
For any fixed predictor $h$, there can be multiple choices of $\vc{s}, \bm{\xi}, \bm{\eta}$ such that $h$ is in the corresponding hypothesis class. 
In the following, we discuss strategies to search for suitable hyper-parameters that can provide tighter generalization bounds. 
To do so, one can instantiate a grid of candidate values for each hyper-parameter that is independent of data. 
Let the grid sizes be $(T_{\vc{s}}, T_{\bm{\xi}}, T_{\bm{\eta}}, T_{\bm{\epsilon}})$, respectively.
We then instantiate the generalization bound in \Cref{theorem: sparsity-aware-single} for each choice of hyper-parameters in the cartesian product of grids with a reduced failure probability $\delta_{\mathrm{red}} = \frac{\delta}{T_{\vc{s}} T_{\bm{\xi}} T_{\bm{\eta}} T_{\bm{\epsilon}}}$.
By a simple union-bound argument, all these bounds hold simultaneously with probability $(1-\delta)$. 
In this way, 
for a fixed $\delta$, the statistical cost above is $\sqrt{\log(T_{\vc{s}} T_{\bm{\xi}} T_{\bm{\eta}} T_{\bm{\epsilon}})}$ as the failure probability dependence in \Cref{theorem: sparsity-aware-single} is $\sqrt{\log\left(\frac{1}{\delta_{\mathrm{red}}}\right)}$. 
The computational cost of a na\"ive search is $\mathcal{O}(T_{\vc{s}} T_{\bm{\xi}} T_{\bm{\eta}} T_{\bm{\epsilon}})$. 
In particular, for multilayer networks, to exhaustively search for a sparsity vector requires a grid of size $T_{\vc{s}} := \prod_{k=1}^{K} d_k$ rendering the search infeasible. Nonetheless, we shall soon show that by employing a greedy algorithm one can still obtain tighter generalization bounds with significantly lesser computational cost. Moreover, these hyper-parameters are not independent, and so we briefly describe here how this optimization can be performed with manageable complexity.

\paragraph{Norm Hyper-parameters $(\bm{\xi},\bm{\eta})$:}  
One can choose $(\bm{\xi}, \bm{\eta})$ from a grid (fixed in advance) of candidate values, to closely match the true properties of the predictor. 
For networks with zero bias, w.l.o.g. one can normalize each layer weight $\W_k \rightarrow \Wtil_k:= \frac{1}{\norm{\W_k}_{(d_k-1,s_{k-1})}} \W_k$ to ensure that $\norm{\Wtil_k}_{(d^{k}-1, s_{k-1})} = 1$ without changing the prediction\footnote{This is not true for networks with non-zero bias. In networks with bias, one can still employ a grid search like in \cite{Bartlett2017SpectrallynormalizedMB}.}. 
The predicted labels, babel function, sparse local radius , margin and the generalization bound in \Cref{theorem: sparsity-aware-single} are all invariant to such a scaling. For the normalized network we can simply let $\xi_k := 1$ for all $k$. Fixing $\bm{\xi}$ this way results in no statistical or computational cost (beyond normalization). 
For discretizing $\bm{\eta}$, we can leverage the fact that for all $(s_k, s_{k-1})$, the reduced babel function is always less than $d_k-s_k-1$ -- since the inner products are scaled by the square of the sparse norms. Thus, we can construct a grid in $[0,1]$ with $T_\eta$ elements, which can be searched efficiently (see \Cref{app:hyper-search} for further details).

\paragraph{Sparsity parameter $\vc{s}$:}
The sparsity vector $\vc{s}$ determines the degree of structure at which we evaluate the generalization of a fixed predictor. For a fixed predictor and relative sensitivity vector $\bm{\epsilon}$, a good choice of $\vc{s}$ is one that has sufficiently large sparse local radii on the training sample resulting in small average sparse loss, $\frac{1}{m} \sum_{\x^{(i)} \in \samp_T} \ell_{\mathrm{sparse}}(h, \x^{(i)})$.  
At the trivial choice of sparsity $\vc{s}=\mathbf{0}$, 
for any choice of $\bm{\epsilon}$, the above loss is exactly zero.
In general, at a fixed $\bm{\epsilon}$, this loss increases with larger (entrywise) $\vc{s}$. 
At the same time, the empirical margin loss term $\hat{R}_{4\zeta_{K+1}\gamma_{K+1}}(h)$ decreases with increasing $\vc{s}$ (since  $\zeta_{K+1}$ grows).
This reflects an inherent tradeoff in the choice of $(\vc{s}, \bm{\epsilon})$ to balance the margin loss and the sparse loss (in addition to the KL-divergence).

For any $\bm{\epsilon}$ and a data point $\z=(\x,y)$, we employ a greedy algorithm to find a
sparsity vector $s^*(\x, \bm{\epsilon})$ in a layer wise fashion such that the loss incurred is zero, i.e. so that $r_k(h,\x) \geq 3\gamma_k$ for all $k$.
At each layer, we simply take the maximum sparsity level with sufficiently large radius. 
The computational cost of such an approach  is $\log_2\left(\prod_{k=1}^{K} d_k\right)$. 
One can thus collect the sparsity vectors $s^*(\x, \bm{\epsilon})$ across the training set and choose the one with sample-wise minimum, so that the average sparse loss vanishes.
Of course, one does not necessarily need the sparse loss to vanish; one can instead choose $\vc{s}$ simply to \emph{control} the sparse loss to a level of $\frac{\alpha}{\sqrt{m}}$. We expand in \Cref{app:hyper-search} how this can done.

\paragraph{Sensitivity vector $\bm{\epsilon}$:} Lastly, the relative sensitivity vector $\bm{\epsilon}$ represents the size of the posterior and desired level of sensitivity in layer outputs upon parameter perturbations. 
Since $\epsilon_k$ denotes \textit{relative perturbation} we can simply let it be the same across all layers. i.e.  $\bm{\epsilon} = \epsilon \cdot [1, \ldots, 1]$.

In summary, as we expand in \Cref{app:hyper-search}, we can compute a best in-grid generalization bound in $\mathcal{O}\left(T_{\bm{\epsilon}} \cdot \log_2\left(\prod_{k=1}^K d_k\right) \cdot \log_2(T_{\bm{\eta}}) \cdot (\sum_{k=1}^K d_k d_{k-1}) \right).$

\section{Numerical Experiments}\label{sec:Experiments}

In this last section we intend to demonstrate the derived bounds on a series of feed-forward networks, of varying width and depth, on MNIST. As we now show, the resulting bounds are controlled and sometimes non-vacuous upon the optimization over a discrete grid for hyper-parameters, as explained above. 

\paragraph{Experimental Setup:}
We train feed-forward networks $h$ with weights $\{\W_k\}_{k=1}^{K+1}$ where $\W_k \in \mathbb{R}^{d_k \times d_{k-1}}$ using the cross-entropy loss with stochastic gradient descent (SGD) for 5,000 steps with a batch size of 100 and learning rate of $0.01$. The MNIST training set is randomly split into train and validation data (55,000~:~5,000). The models are optimized on the training data and the resulting measures are computed on validation data. 
To evaluate scaling with the number of samples, $m$, we train networks on randomly sampled subsets of the training data of increasing sizes from 20\% to 100\% of the  training set. Because of the chosen architectures, all of these models are over-parametrized (i.e. having more parameters than training samples).

Recall that the bound on generalization error in \Cref{theorem: sparsity-aware-single} depends on the KL divergence between a posterior centered at trained predictor $h$, $\mathcal{N}(h, \bm{\sigma^2}_{\mathrm{sparse}})$, and the prior $P = \mathcal{N}(h_{\mathrm{prior}}, \bm{\sigma^2}_{\mathrm{sparse}})$. Thus, each model is encouraged to be close to its initialization via a regularization term. In this way, we minimize the following regularized empirical risk based on the cross-entropy loss as well as 
a regularization term with penalty $\lambda$ (set as $\lambda = 1.0$ for all experiments for simplicity),
\begin{equation*}
\min_{\{\W_k\}_{k=1}^{K+1}} ~~ \frac{1}{m}\sum_{i=1}^m \ell_{\mathrm{cross-ent}} \Big(h, (\x_{i},y_{i})\Big) + \frac{\lambda}{K+1}\sum_{k=1}^{K+1} \norm{\W_k - \W_{\mathrm{prior},k}}^2_F. 
\end{equation*}
\paragraph{Choice of Prior:}
As with any PAC-Bayes bound, choosing a prior distribution with an appropriate inductive bias is important.  
For example, optimizing the choice of prior by instantiating multiple priors simultaneously was shown to be an effective procedure to obtain good generalization bounds \citep{Langford2001NotBT, Dziugaite2017ComputingNG}.
In this work, we evaluate our bounds for two choices of the prior: \emph{a)} a data-independent prior, $P_0:= \mathcal{N}(h_{\mathbf{0}}, \bm{\sigma^2}_{\mathrm{sparse}})$ centered at a model with zero weights, $h_{\mathbf{0}}$; and \emph{b)} a data-dependent prior $P_{\mathrm{data}}:= \mathcal{N}(h_{\mathrm{init}}, \bm{\sigma^2}_{\mathrm{sparse}})$ centered at a model $h_{\mathrm{init}}$ obtained by training on a small fraction of the training data (5\% of all training data). 
Note that this choice is valid, as the base hyper-parameter $(\vc{s}, \bm{\xi}, \bm{\eta}, \bm{\epsilon})$ are chosen independent of data, and the empirical risk terms in the bound are not evaluated on the small subset of data $h_{\mathrm{init}}$ is trained on.

\paragraph{Generalization bounds across width:}
We first train a 2-layer (1 hidden layer) fully connected neural network with increasing widths, from 100 to 1,000 neurons. Note that in all cases these models are over-parametrized.  
In \Cref{fig:gen100,fig:gen500,fig:gen1000} we plot the true risk (orange curve) and the generalization bounds (blue curve) from \Cref{theorem: sparsity-aware-single} across different sizes of training data and for the two choices of priors mentioned above. 
We observe that our analysis, when coupled with data-dependent prior $P_{\mathrm{data}}$, generates non-vacuous bounds for a network with width of 100. Even for the na\"ive choice of the prior $P_0$, the bound is controlled and close to 1. Furthermore, note that our bounds remain controlled for larger widths. In \Cref{app:exp}, we include complementary results depicting our generalization bounds for 3-layer networks.

\begin{figure}
\centering
\subcaptionbox{Model width: 100.\label{fig:gen100}}
{\includegraphics[width=0.495\textwidth]{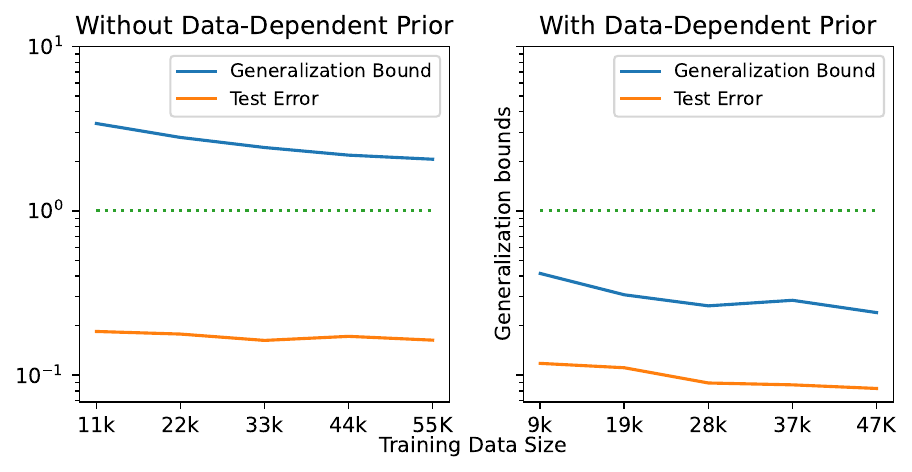}}
\subcaptionbox{Model width: 500.\label{fig:gen500}}
{\includegraphics[width=0.495\textwidth]{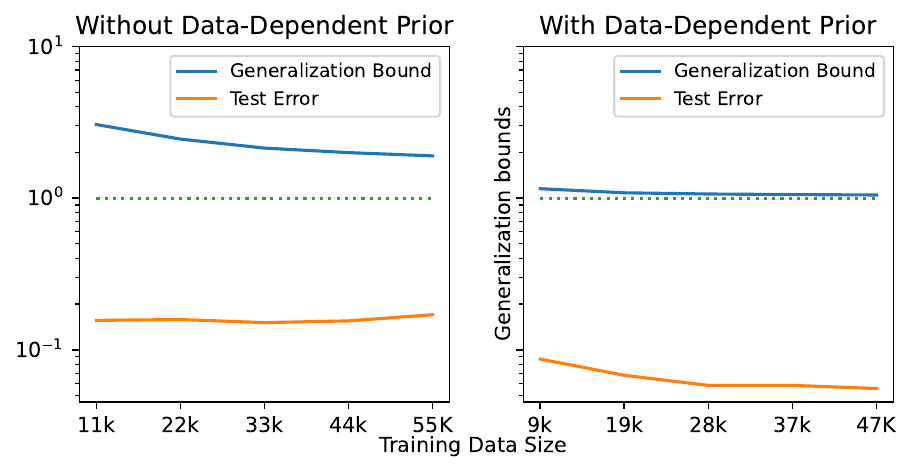}}
\\[5pt]
\subcaptionbox{Model width: 1,000. \label{fig:gen1000}}
{\includegraphics[width=0.5\textwidth]{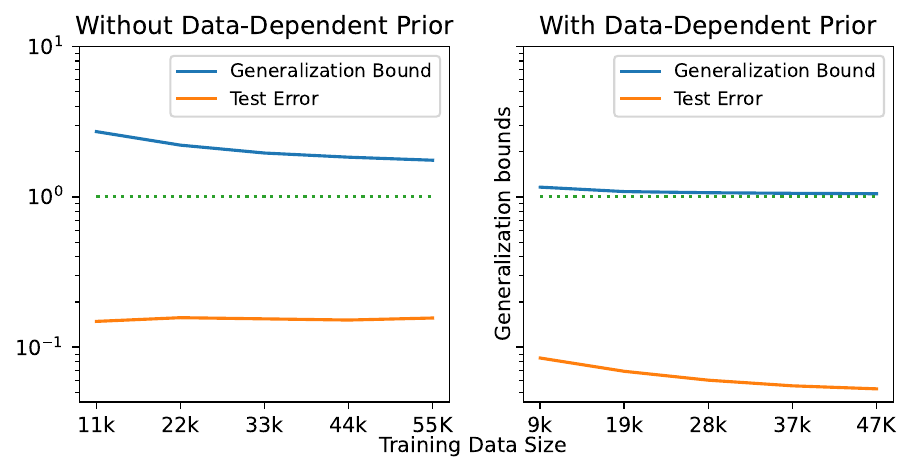}}
\caption{Generalization error of a 2-layer model of different widths trained on MNIST.}\label{fig:generalization_bounds}
\end{figure}


\paragraph{Effective Activity ratio:}
Lastly, we intend to illustrate the degree of sparsity achieved in the obtained models that allow for the bounds presented in \Cref{fig:generalization_bounds}. For each data point $\x$ and relative perturbation level $\epsilon$, we define the Effective Activity ratio $\kappa(\x,\epsilon):= \frac{\sum_{k} (d_k-s_k)(d_{k-1}-s_{k-1})}{\sum_k d_k d_{k-1}}$ where $\vc{s} = s^*(\x, \epsilon)$, the greedy sparsity vector chosen such that the sparse loss in \Cref{theorem: sparsity-aware-single} is zero. 
In this way, $\kappa(\x,\epsilon)$ measures the reduced local dimensionality of the model at  input $\x$ under perturbations of relative size $\epsilon$. When $\kappa(\x,\epsilon) = 1$, there are no sparse activation patterns that are stable under perturbations, and the full model is considered at that point. On the other hand, when $0<\kappa(\x, \epsilon) \ll 1$, the size of stable sparse activation patterns $s^*(\x,\epsilon)_k$ at each layer is close to the layer dimension $d_k$. 
\Cref{theorem: sparsity-aware-single} enables a theory of generalization that accounts for this local reduced dimensionality. 

We present the effective activity rations for a trained 3-layer model in \Cref{fig:3layer}, and include the corresponding results for the 2-layer model in \Cref{app:exp} for completeness. The central observation from these results is that trained networks with larger width have \emph{smaller} effective activity ratios across the training data.
In \Cref{fig:HEA3} (as well as in \Cref{fig:HEA2} for the 2-layer model), the distribution of effective activity ratio across the training data at $\epsilon=10^{-4}$ shows that smaller width networks have less stable sparsity. 
In turn, \Cref{fig:AEA3} and \Cref{fig:AEA2} demonstrate that this effect is stronger for smaller relative perturbation levels. 
This observation is likely the central reason of why our generalization bounds do not increase drastically with model size.


\begin{figure}
    \centering
\begin{minipage}{0.49\textwidth}
\includegraphics[trim = 0 0 0 40, clip, width=\textwidth]{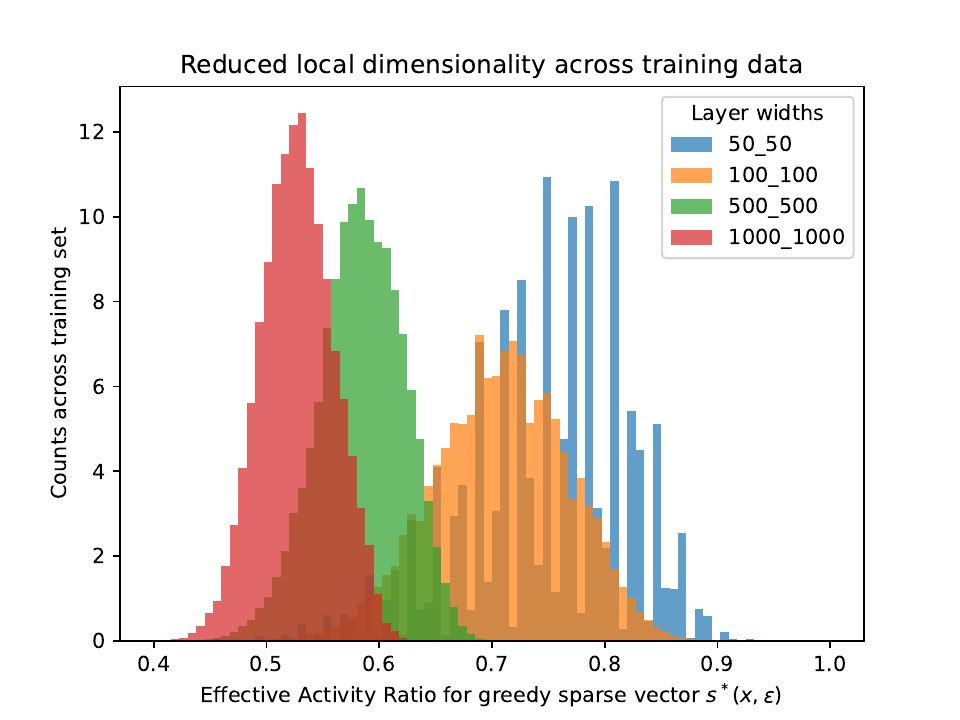}
\subcaption{Histogram of Effective Activity Ratio at $\epsilon = 10^{-4}$}
\label{fig:HEA3}
\end{minipage}%
\begin{minipage}{0.49\textwidth}
\includegraphics[trim = 0 0 0 40, clip, width=\textwidth]{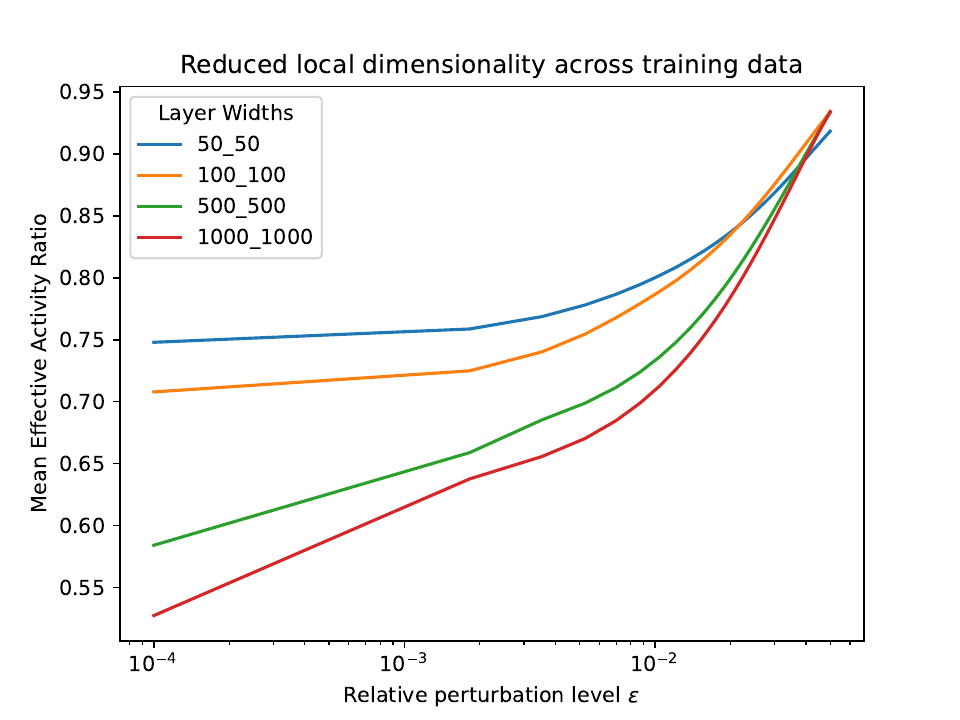}
\subcaption{Average Effective Activity Ratio}
\label{fig:AEA3}
\end{minipage}
\caption{Effective activity ratio $\kappa(\x, \epsilon)$ based on greedy sparsity vector $s^{*}(\x,\epsilon)$ for 3-layer networks  (smaller implies sparser stable activations).}
\label{fig:3layer}
\end{figure}

\section{Conclusion}\label{sec:Conclusion}

This work makes explicit use of the degree of sparsity that is achieved by ReLU feed-forward networks, reflecting the level of structure present in data-driven models, but without making any strong distributional assumptions on the data. Sparse activations imply that only a subset of the network is active at a given point. By studying the stability of these local sub-networks, and employing tools of derandomized PAC-Bayes analysis, we are able to provide bounds that exploit this effective reduced dimensionality of the predictors, as well as avoiding exponential dependence on the sensitivity of the function and of depth. Our empirical validation on MNIST illustrates our results, which are always controlled and sometimes result in non-vacuous bounds on the test error. Note that our strategy to instantiate our bound for practical models relied on a discretization of the space of hyper-parameters and a greedy selection of these values. This is likely suboptimal, and the grid of hyper-parameters could be further tuned for each model. Moreover, in light of the works in \citep{Dziugaite2017ComputingNG, Dziugaite2018EntropySGDOT, Zhou2019NonvacuousGB}, we envision optimizing our bounds directly, leading to even tighter solutions.

\section*{Acknowledgments}
We kindly thank Vaishnavh Nagarajan for helpful conversations that motivated the use of de-randomized PAC-Bayesian analysis. This work was supported by NSF grant CCF 2007649.

\bibliographystyle{plainnat}
\bibliography{references}
\newpage 
\appendix

\section{Missing Proofs}
\subsection{Sparsity in feed-forward maps}
In this subsection we provide explicit proofs for all theorems corresponding to stability of index sets and reduced size (such as \Cref{lemma: slr,proposition: reduced-size-multi}) and sensitivity of outputs (such as \Cref{proposition: reduced-sensitivity})

\subsubsection{Stability of Index sets in a Single layer Feed-forward Map}\label{app: lemma: slr}
\begin{proof}(For \Cref{lemma: slr})
    To prove the first statement, note that for all $i\in  I(\W, \x, s_1)$, 
    \begin{equation}\label{eq: index-rsl}
    \max\left\{0, - \frac{\Wvec{}{i} \cdot \x}{\xi_1\cdot \zeta_0} \right\} \geq \rsl(\W, \x, s_1) 
    \end{equation}
    Hence if $\rsl(\W, \x, s_1) > 0$, then $I(\W, \x, s_1)$ is inactive.  
    Now consider any perturbed weight $\What$ and any perturbed input $\xhat$ such that $\norm{\xhat-\x} \leq d_0-s_0$. The absolute difference between the normalized pre-activation values at each index can be bounded, 
    \begin{align*}
    &\left\lvert \hat{\vc{w}}[i] \cdot \xhat - \Wvec{}{i} \cdot\x \right\rvert  \\
    &=  \left\lvert \;\Wvec{}{i} \cdot\left(\xhat - \x\right) 
     + \left(\hat{\vc{w}}[i]  - \Wvec{}{i} \right) \cdot \x \
     + \left(\hat{\vc{w}}[i]  - \Wvec{}{i} \right) \cdot (\xhat-\x)\; \right\rvert\\
    &\leq \underset{|J_{0}| = d_0-s_0}{\max} \left(\norm{\mathcal{P}_{J_0}(\vc{w}[i])}_2 \cdot \norm{\xhat-\x}_2 + \norm{\mathcal{P}_{J_0}\left(\hat{\vc{w}}[i]-\vc{w}[i]\right)}_2 \norm{\x}_2 + \norm{\mathcal{P}_{J_{0}}\left(\hat{\vc{w}}[i]-\vc{w}[i]\right)}_2 \norm{\xhat-\x}_2\right).\\
    &\leq
    \xi_1\cdot \norm{\xhat-\x}_2  
    + \norm{\What-\W}_{(d_1-1,s_0)} \cdot \zeta_0
    + \norm{\What-\W}_{(d_1-1,s_0)} \cdot \norm{\xhat-\x}_2  \\
    &= \xi_1\cdot \zeta_0 \cdot \left(
    -1 + \left(1 +  \frac{\norm{\xhat-\x}_2}{\zeta_0} \right)\left(1+\frac{\norm{\What-\W}_{(d_1-1,s_0)}}{\xi_1}
    \right) \right)
    \end{align*}
    The above inequalities show that, 
\begin{equation}\label{eq: index-rsl-2}
\left\vert  \frac{
\vc{w}[i] \cdot \x - \hat{\vc{w}}[i] \cdot \xhat}
{\xi_1\cdot \zeta_0} \right\rvert 
\leq 
-1 + \left(1 + \frac{\norm{\xhat-\x}_2}{\zeta_0}\right)
\left(1 + \frac{\norm{\What-\W}_{(d_1-1,s_0)}}{\xi_1}\right)
\end{equation}
This proves the second statement by plugging in the bounds on the relative perturbation terms above and using Equation \eqref{eq: index-rsl} to note that,
    \begin{align*}
    &\rsl(\W, \x, s_1) > \left(
    -1 + \left(1 +  \frac{\norm{\xhat-\x}_2}{\zeta_0} \right)\left(1+\frac{\norm{\What-\W}_{(d_1-1,s_0)}}{\xi_1}\right)
    \right) \\
    &\implies \;\forall\; i\in I(\W,\x,s_1), \;\;
   \frac{-\vc{w}[i]\cdot \x}{ \xi_1 \zeta_0} > \left(
    -1 + \left(1 +  \frac{\norm{\xhat-\x}_2}{\zeta_0} \right)\left(1+\frac{\norm{\What-\W}_{(d_1-1,s_0)}}{\xi_1}\right)
    \right) 
    \\
    &\implies \;\forall\; i\in I(\W,\x,s_1), \;\;
    \frac{-\hat{\vc{w}}[i]\cdot \xhat}{ \xi_1 \zeta_0} > 0
    \end{align*}
    Thus when the sparse local radius is large as stated, the index set $I(\W,\x,s_1)$ is inactive for $\hat{\Phi}(\xhat)$. As a special case, when $\epsilon=0$, the same logic implies that the index set is inactive for $\Phi(\xhat)$.
    It is left to prove the stability of the sparse local radii.

For the final statement, recall the definitions of sparse local radius at  $(\W, \x)$ and $(\What, \xhat)$ respectively 
\begin{align*}
\rsl(\W, \x, s_1) 
&:= \act{
\textsc{sort}\left(-\left[ \frac{  \Wvec{}{i} \cdot \vc{x} }{
\xi_1\zeta_0
}  \right]_{i=1}^{d_1},\; s_1\right) },\\
\rsl(\What, \xhat, s_1) 
&:= \act{
\textsc{sort}\left(-\left[ \frac{ 
\hat{\vc{w}}[i] \cdot \xhat }
{
\xi_1\zeta_0
}  \right]_{i=1}^{d_1},\; s_1\right) }.
\end{align*} 
Since $\relu$ is 1-Lipschitz, hence the distance between the radius measurements can be bounded as,
\begin{align*}
&|\rsl(\What, \xhat, s_1)  - \rsl(\W, \x, s_1)| \\
&\leq 
\left \lvert 
\textsc{sort}\left(-\left[ \frac{ 
\hat{\vc{w}}[i] \cdot \xhat }
{
\xi_1\zeta_0
}  \right]_{i=1}^{d_1},\; s_1\right)
- \textsc{sort}\left(-\left[ \frac{ 
{\vc{w}}[i] \cdot \xhat }
{
\xi_1\zeta_0
}  \right]_{i=1}^{d_1},\; s_1\right)
\right \rvert 
\end{align*}
Then observe that, 
\begin{align*}
&\textsc{sort}\left(-\left[ \frac{ 
\hat{\vc{w}}[i] \cdot \xhat }
{
\xi_1\zeta_0
}  \right]_{i=1}^{d_1},\; s_1\right)\\
&= \max_{\hat{I} \subset [d_1], |\hat{I}|=s_1}~
\min_{i \in \hat{I}}
\frac{- \hat{\vc{w}}[i]\cdot \xhat}{\xi_1\zeta_0}   \\
&\geq 
\min_{i \in I(\W, \x, s_1)} 
\frac{- \hat{\vc{w}}[i]\cdot \xhat}{\xi_1\zeta_0}  
 \\
&\geq 
\min_{i \in I(\W, \x, s_1)} 
\frac{- \vc{w}[i]\cdot \xhat}{\xi_1\zeta_0}  
- 
\left(
    -1 + \left(1 +  \frac{\norm{\xhat-\x}_2}{\zeta_0} \right)\left(1+\frac{\norm{\What-\W}_{(d_1-1,s_0)}}{\xi_1}\right)
    \right) 
\quad \text{by Eq. \eqref{eq: index-rsl-2}}.\\
&= \textsc{sort}\left(-\left[ \frac{ 
{\vc{w}}[i] \cdot \xhat }
{
\xi_1\zeta_0
}  \right]_{i=1}^{d_1},\; s_1\right) - 
\left(
    -1 + \left(1 +  \frac{\norm{\xhat-\x}_2}{\zeta_0} \right)\left(1+\frac{\norm{\What-\W}_{(d_1-1,s_0)}}{\xi_1}\right)
    \right) 
\end{align*}
By repeating the same arguments, one can establish that, 
\begin{align*}
&\left \lvert 
\textsc{sort}\left(-\left[ \frac{ 
\hat{\vc{w}}[i] \cdot \xhat }
{
\xi_1\zeta_0
}  \right]_{i=1}^{d_1},\; s_1\right)
- \textsc{sort}\left(-\left[ \frac{ 
{\vc{w}}[i] \cdot \xhat }
{
\xi_1\zeta_0
}  \right]_{i=1}^{d_1},\; s_1\right)
\right \rvert \\
&\leq \left(
    -1 + \left(1 +  \frac{\norm{\xhat-\x}_2}{\zeta_0} \right)\left(1+\frac{\norm{\What-\W}_{(d_1-1,s_0)}}{\xi_1}\right)
    \right) 
\end{align*}
Hence the difference between the sparse local radii are bounded as required. 

Lastly, to show the reduced sensitivity of the predictor, notice the following. 
Let $I_0$ be an inactive index set in the input of size $s_0$ and let $J_0 := (I_0)^c$ be its complement. 
When $\rsl(\W, \x, s_1) > 0$, the index set $I(\W,\x,s_1)$ is inactive. Let $J_1:= (I(\W,\x,s_1))^c$ be its complement index set. Then,
\[
\Phi(\x) = \act{\W \x} = \act{\mathcal{P}_{J_1,J_0}(\W) \mathcal{P}_{J_0}(\x)}
\]
Hence, $\norm{\Phi(\x)}_2 \leq \norm{\mathcal{P}_{J_1, J_0}(\W)}_2 \norm{\x}_2 \leq \norm{\W}_{(s_1,s_0)} \zeta_0 \leq \xi_1\sqrt{1+\eta_1}\zeta_0$. Thus proving the reduced size of the outputs. 
When $\rsl(\W,\x,s_1) > -1+(1+\epsilon_0)(1+\epsilon_1)$, for perturbed inputs and weights as described, the index set $I(\W,\x,s_0)$ is inactive for $\Phi(\xhat)$ and $\hat{\Phi}(\xhat)$. Again let $J_1$, $J_0$ be the complement sets,
\begin{align*}
    \norm{\hat{\Phi}(\xhat) - \Phi(\x)}_2 
    & = \norm{\act{\What \xhat} - \act{\W \x}}_2\\
    & = \norm{\act{\mathcal{P}_{J_1, J_0}(\What) \mathcal{P}_{J_0}(\xhat)} - \act{\mathcal{P}_{J_1, J_0}(\W) \mathcal{P}_{J_0}(\x)}}_2\\
    & \leq \norm{\mathcal{P}_{J_1, J_0}(\What) \mathcal{P}_{J_0}(\xhat) - \mathcal{P}_{J_1, J_0}(\W) \mathcal{P}_{J_0}(\x)}_2\\
    & \leq \norm{\mathcal{P}_{J_1, J_0}(\What) \cdot \mathcal{P}_{J_0}(\xhat-\x)
    + \mathcal{P}_{J_1, J_0}(\What-\W)\cdot \mathcal{P}_{J_0}(\x)
    }_2\\
    & \leq \norm{\mathcal{P}_{J_1, J_0}(\What) \cdot \mathcal{P}_{J_0}(\xhat-\x)
    }_2 + \norm{\mathcal{P}_{J_1, J_0}(\What-\W)\cdot \mathcal{P}_{J_0}(\x)
    }_2\\
    & \leq (\norm{\mathcal{P}_{J_1, J_0}(\W)}_2 + \norm{\mathcal{P}_{J_1, J_0}(\What-\W)}_2)\cdot \epsilon_0 \zeta_0 + \epsilon_1 \xi_1\sqrt{1+\eta_1} \cdot \zeta_0.\\
    &\leq \left(\epsilon_0 + \epsilon_0\epsilon_1 + \epsilon_1\right)\cdot \xi_1\sqrt{1+\eta_1}\zeta_0\\
    &= \left(-1 + (1+\epsilon_0)(1+\epsilon_1)\right)\cdot \xi_1\sqrt{1+\eta_1}\zeta_0
\end{align*}

\end{proof}

\subsubsection{Reduced size of layer outputs in Multilayer networks}\label{app: proposition: reduced-size-multi}
Consider the layer-wise input domains, 
$
\cX_k := \{\vt \in \mathbb{R}^{d_k} ~|~ \norm{\vt}_2 \leq \zeta_k, \; \norm{\vt}_0 \leq d_k - s_k
\}$. 
\begin{proof}(For \Cref{proposition: reduced-size-multi})

From \Cref{lemma: slr}, $r_1(h,\x_0) >0$ guarantees existence of inactive index set $I_1(h,\x_0)$ and a reduced size of the output such that $\norm{\x_1} \leq \mathsf{M}_{\cX} \norm{\W}_{(s_1,0)}$. 
From \Cref{lemma: bound-submatrix-norm} and the definition of the hyper-parameters $\xi_1$ and $\eta_1$, 
\[
\norm{\W}_{(s_1,0)} \leq \norm{\W}_{d_1-1,0} \sqrt{1+\mu_{s_1,0}(\W)} \leq \xi_1 \sqrt{1+\eta_1}.\]
Hence $\norm{\x}_1 \leq \zeta_1$. 
Thus the statement of the theorem is true for $k=1$.

Assume that the statement is true for all layers $1\leq n\leq k$. 
Hence when $r_n(h, \x_0) > 0$ for all layers $1\leq n \leq k$, there exists index sets $I_1(h,\x_0), \ldots, I_k(h,\x_0)$ such that $\mathcal{P}_{I_n(h,\x_0)}(\x_n) = \mathbf{0}$ and $\norm{\x_n} \leq \zeta_n$ for all $1\leq n \leq k$. Thus $\x_n \in \mathcal{X}_n$ for all $1\leq n \leq k$.

If additionally $r_{k+1}(h,\x_0) > 0$, then by invoking \Cref{lemma: slr} for input $\x_k \in \cX_k$ and weight $\W_{k+1} \in \cW_{k+1}$, we see that $I_{k+1}(h,\x_0)$ is inactive for $\x_{k+1}$ and further \Cref{lemma: slr} shows that $\norm{\x_{k+1}} \leq \xi_{k+1}\sqrt{1+\eta_{k+1}} \cdot \zeta_k = \zeta_{k+1}$ as desired.
Hence the theorem is true for all $1\leq k\leq K$. 

For the final layer we note that since $I_K(h,\x_0)$ of size $s_K$ is inactive for $\x_K$, 
\[
\norm{h(\x_0)}_{\infty} \leq \norm{\mathcal{P}_{[C], J_K}(\W_{K+1})}_{2\rightarrow \infty} \norm{\x_K}_2 \leq \norm{\W_{K+1}}_{C-1, s_{K}} \norm{\x_K} \leq \xi_{K+1} \cdot \zeta_{K} = \zeta_{K+1}.
\] 
In the above inequality, we have used the fact that for any matrix the reduced $2\rightarrow \infty$ norm, 
\begin{align*}
\norm{\mathcal{P}_{[C], J_K}(\W_{K+1})}_{2\rightarrow \infty} 
\leq \max_{|J|=d_K-s_K} \max_{j\in C} \norm{\mathcal{P}_{J}(\vc{w}_{K+1}[j])}_{2} 
= \norm{\W_{K+1}}_{C-1, s_{K}}.
\end{align*}
\end{proof}

\subsubsection{Reduced sensitivity of layer outputs in Multilayer networks}\label{app: proposition: reduced-sensitivity}
\begin{proof}(For \Cref{proposition: reduced-sensitivity})
From \Cref{lemma: slr}, $r_1(h,\x_0) > \gamma_1$ guarantees existence of inactive index set $I_1(h,\x_0)$ such that $\mathcal{P}_{I_1(h,\x_0)}(\xhat_1) = \mathcal{P}_{I_1(h,\x_0)}(\x_1) = \mathbf{0}$. 
Further from \Cref{lemma: slr} (with input perturbation $\epsilon_0 = 0$), the distance between the first layer representations are bounded as $\norm{\xhat_1-\x_1} \leq \left(-1 + (1+\epsilon_0)(1+\epsilon_1) \right) \norm{\W}_{(s_1,0)} \mathsf{M}_{\cX} \leq \epsilon_1 \cdot \xi_1 \sqrt{1+\eta_1}\zeta_0 = \zeta_1 \gamma_1$. 
Thus the statement of the theorem is true for $k=1$.

Assume that the statement is true for all layers $1\leq n\leq k$. 
Thus there exists index sets $I_1(h,\x_0), \ldots, I_k(h,\x_0)$ such that $\mathcal{P}_{I_n(h,\x_0)}(\xhat_n) = \mathcal{P}_{I_n(h,\x_0)}(\x_n) = \mathbf{0}$ and the distance between the layer representations are bounded $\norm{\xhat_n -\x_n} \leq \zeta_n \cdot \gamma_n$ for all $1\leq n\leq k$.

From \Cref{proposition: reduced-size-multi}, due to the reduced size, $\x_{k} \in \cX_k$ and the sparse local radius $r_{k+1}(h,\x_0) \in [0,1]$. 
For the perturbed input to layer $k+1$, $\xhat_k$ we note that $\norm{\xhat_k-\x_k}_2 \leq \zeta_k \cdot \gamma_k$ and $\norm{\xhat_k-\x_k}_0 \leq d_k-s_k$. 
The perturbed weight $\What_{k+1}$ is such that $\norm{\What_{k+1} - \W_{k+1}}_{d_{k+1}-1, s_{k}} \leq \epsilon_{k+1} \xi_{k+1}$.
Hence applying \Cref{lemma: slr} on inputs $\x_{k} \in \cX_k$ and weight $\What_{k+1} \in \cW_{k+1}$ shows that 
the set $I_{k+1}(h,\x_0)$ is inactive for $\xhat_{k+1}$ 
and further, 
\[
\norm{\xhat_{k+1}-\xhat_k}_2 \leq \underbrace{\left(-1 + (1+\gamma_k)(1+\epsilon_{k+1})\right)}_{=: \gamma_{k+1}} \underbrace{\xi_{k+1} \sqrt{1+\eta_{k+1}} \cdot \zeta_k}_{=: \zeta_{k+1}}. 
\]
Hence the conclusion follows for all layers $1\leq k\leq K$. For the last layer, under the assumption that $r_k(h,\x_0) > \gamma_k$ for all $k$, we know that $I_{K}(h,\x_0)$ of size $s_K$ is inactive for both $\x_K$ and $\xhat_K$. Let $J_K$ be the complement of $I_{K}(h,\x_0)$. We can bound the distance between the network outputs as follows,
\begin{align*}
\norm{\hat{h}(\x_0) - h(\x_0)}_{\infty} 
&= \norm{\What_{K+1}\xhat_K - \W_{K+1}\x_K}_{\infty}\\
&= \norm{\mathcal{P}_{[C], J_K}(\What_{K+1})\mathcal{P}_{J_K}(\xhat_K) -  \mathcal{P}_{[C], J_K}(\W_{K+1}) \mathcal{P}_{J_K}(\x_K)}_{\infty}\\
& \leq \norm{\mathcal{P}_{[C], J_K}(\What) \cdot \mathcal{P}_{J_K}(\xhat_K-\x_K)
    + \mathcal{P}_{[C], J_K}(\What-\W)\cdot \mathcal{P}_{J_K}(\x_K)
    }_{\infty}\\
    & \leq \norm{\mathcal{P}_{[C], J_K}(\What) \cdot \mathcal{P}_{J_K}(\xhat_K-\x_K)
    }_{\infty} + \norm{\mathcal{P}_{[C], J_K}(\What-\W)\cdot \mathcal{P}_{J_K}(\x_K)
    }_{\infty}\\
    & \leq \left(\norm{\mathcal{P}_{[C], J_K}(\W)}_{2\rightarrow \infty} + \norm{\mathcal{P}_{[C], J_K}(\What-\W)}_{2\rightarrow \infty}\right)\cdot \norm{\mathcal{P}_{J_K}(\xhat-\x)}_2 \\
    & \qquad + \norm{\mathcal{P}_{[C], J_K}(\What-\W)}_{2\rightarrow \infty} \cdot \norm{\mathcal{P}_{J_K}(\x_K)}_2\\
    &\leq 
    (1+ \epsilon_{K+1} ) \underbrace{\xi_{K+1}\cdot \zeta_K}_{=\zeta_{K+1}}\gamma_K + \epsilon_{K+1} \underbrace{\xi_{K+1}\zeta_K}_{=\zeta_{K+1}}\\
    &= \zeta_{K+1} \cdot \underbrace{\left( (1+ \epsilon_{K+1})\gamma_K + \epsilon_{K+1}\right)}_{=\gamma_{K+1}}\\
    &= \zeta_{K+1} \cdot \gamma_{K+1}.
\end{align*}
\end{proof}

\subsection{Gaussian Sensitivity Analysis}
In this subsection we seek to understand the probability that $\hat{h} \in \mathcal{B}(h, \bm{\epsilon})$ when the perturbed layer weights are randomly sampled from $\mathcal{N}(h, \bm{\sigma}^2)$. 
For any failure probability $\delta > 0$, we define layer-wise normalization functions $\alpha^k, \beta^k : [0,1] \rightarrow \mathbb{R}^{>0}$ that are dimension-dependent (but data/weights independent),
\begin{align}\label{eq: alpha-normalization}
\alpha^k(s_k, s_{k-1}, \delta) 
&:= 
\left(\sqrt{d_k-s_k}+\sqrt{d_{k-1}-s_{k-1}}) + \sqrt{2\log\binom{d_k}{s_k} + 2\log\binom{d_{k-1}}{s_{k-1}} + 2\log\left(\frac{1}{\delta}\right)}\right)
\end{align}
We can now bound the probability that $\hat{h} \in \mathcal{B}(h, \vc{s}, \bm{\epsilon})$ when constructed using Gaussian perturbations. 
\begin{lemma}\label[lemma]{lemma: random-perturb-network}
Define layer-wise variance parameter $\bm{\sigma}(\delta)$ as 
\begin{align*}
\forall\; 1\leq k \leq K, ~~ \sigma^k(\delta) := \epsilon^k \min &\left\{ 
\frac{ 
\xi^k \sqrt{1+\eta^k}
}{
\alpha^k(s_k, s_{k-1}, \frac{\delta}{K+1}) 
 }   
 ,
 \frac{
 \xi^k 
 }{\alpha^k(d_k-1, s_{k-1}, \frac{\delta}{K+1}) 
}  \right\}. 
\end{align*}
and let $\sigma_{K+1}(\delta):= \epsilon_{K+1}\cdot \frac{\epsilon_{K+1}}{\alpha^k(C-1, s_{K}, \frac{\delta}{K+1})}$,
For any $h \in \cH$, with probability at least $(1-\delta)$, a Gaussian perturbed network sampled from $\mathcal{N}(h, \sigma^2(\delta))$ is in the local neighbourhood $\mathcal{B}(h,\bm{\epsilon})$. 
\end{lemma}
\begin{proof}
 As per \Cref{app: lemma: sparse-norm}, with probability at least $(1-\delta)$, a perturbed network $\hat{h}$ sampled from $\mathcal{N}(h, \sigma^2(\delta))$ satisfies the following inequalities simultaneously at every layer, 
\begin{align*}
    \norm{\What_k-\W_k}_{(s_k,s_{k-1})} 
\leq \sigma^k(\delta) \cdot \alpha^k(s_k,s_{k-1}, \delta), \quad   
    \norm{\What_k-\W_k}_{(d_k-1,s_{k-1})} \leq  \sigma^k(\delta) \cdot \alpha^k(d_k-1,s_{k-1},\delta).
\end{align*}
Clearly by the choice of variance parameter this implies that $\hat{h} \in \mathcal{B}(h, \bm{\epsilon})$ with probabilty at least $(1-\delta)$ since, 
\begin{align*}
\max\left\{  \frac{\norm{\What_k-\W_k}_{(s_k,s_{k-1})}}{\xi_k \sqrt{1+\eta_k}} , 
\frac{ \norm{\What_k-\W_k}_{(d_k-1,s_{k-1})} }{\xi_k}
\right\}
\leq \epsilon_k.
\end{align*}
\end{proof}

\subsection{Sparsity-aware generalization theory: Expanded Result}\label{app: sparsity-aware}
In this section we prove a stronger version of the simplified result in \Cref{theorem: sparsity-aware-single}. 
\begin{theorem}\label{theorem: sparsity-aware-expanded}
Let $\mathcal P := \prod_{k=1}^K \mathcal{P}_k$ 
be any (factored) prior distribution over depth-$(K+1)$ feed-forward network 
chosen independently of data. 
Let $h \in \cH$ be any feed-forward network (possibly trained on sample data).  
With probability at least $(1-\delta)$ over the choice of i.i.d training sample $\samp_T$ of size $m$, 
the generalization error of $h$ is bounded as follows,
\begin{align*}
R_{0}(h) 
&\leq  
\hat{R}_{4 \zeta_{K+1} \cdot \gamma_{K+1}}(h)  + \frac{4 (K+1)}{\sqrt{m}-1} 
+  \sqrt{\frac{4\mathrm{KL}\left(  
  \mathcal{N}\left(h, \bm{\sigma}^2_{\mathrm{sparse}}\right)
 \;|| \;
 P
 \right) + 2\log\left(\frac{2m(K+1)}{\delta}\right) 
}{m-1}}
\\
&
 + \sum_{k \in [K]}\frac{1}{m}\sum_{(\x^{(i)},\y^{(i)}) \in \samp_T}  \mathbf{1}\Bigg\{ \exists 1\leq n\leq 
 k,\;  
r_{n}(h, \x^{(i)}) < \gamma_n
\Bigg
\} \\
 & + \sum_{k \in [K]} \frac{1}{m}\sum_{(\x^{(i)},\y^{(i)}) \in \samp_T} \cdot \mathbf{1}\Bigg\{\exists 1\leq n\leq 
 k,\;  
r_{n}(h, \x^{(i)}) < 3 \gamma_n 
\Bigg
\} \\
& 
 + \sum_{k \in [K]} \sqrt{\frac{4 \cdot \sum_{n=1}^k \mathrm{KL}\left(  
  \mathcal{N}\left(\W_n, \sigma^2_{n}\right)
 \;|| \;
 \mathcal{P}_n
 \right) + 2\log\left(\frac{2m(K+1)}{\delta}\right) 
}{m-1}}
\end{align*}
with 
layer-wise variance parameter $\bm{\sigma}_{\mathrm{sparse}} = \{\sigma_{1}, \ldots, \sigma_{K}\}$ 
is defined as,
\begin{align}\label{eq: sigma-posterior}
\sigma_k := \epsilon^k \min &\left\{ 
\frac{ 
\xi^k \sqrt{1+\eta^k}
}{
\alpha^k(s_k, s_{k-1}, \frac{1}{(K+1)\sqrt{m}}) 
 }   
 ,
 \frac{
 \xi^k 
 }{\alpha^k(d_k-1, s_{k-1}, \frac{1}{(K+1)\sqrt{m}) }}  \right\}. 
\end{align}
\end{theorem}
\begin{proof}
We note that $R_\gamma (h) = \expect_{\z \sim \cD} (\ell_\gamma (h, \z)) = \expect_{\z \sim \cD} \big[ \mathbf{1}\left\{ \rho(h,\z) < \gamma \right\}\big]$. For the margin property $\rho(h,\z) := h(\x)_y - \max_{j \neq y} h(\x)_j$ with margin threshold $\gamma$,  
\Cref{lemma: derandomization} shows that with probability at  least $(1-\frac{\delta}{K+1})$ over the choice of i.i.d training sample $\samp_T$ of size $m$, 
for any predictor $h \in \cH$,
the generalization error is bounded by
\begin{align}
&\prob_{\z \sim \cD} \bigg[ \rho(h, \z) < 0 \bigg] \\
\nonumber 
& \leq \frac{1}{m} \sum_{\z^{(i)} \in \samp_T} \mathbf{1}\bigg[ \rho(h, \z^{(i)}) < 4 \zeta_{K+1} \gamma_{K+1} \bigg] + \frac{2}{\sqrt{m}-1} + + \sqrt{\frac{4\mathrm{KL}\left(  
  \mathcal{N}\left(h, \bm{\sigma}^2_{\mathrm{sparse}}\right)
 \;|| \;
 \mathcal{P}
 \right) + 2\log(\frac{2m(K+1)}{\delta}) 
}{2(m-1)}} . 
\\
&
+ \mu_{\samp_T} \left(h, (\rho , 4\zeta_{K+1} \gamma_{K+1})\right)
+ \mu_\cD \left(h, (\rho , 4\zeta_{K+1} \gamma_{K+1})\right)
\end{align}
It remains to bound the term $\mu_\cD \left(h, (\rho , 4 \zeta_{K+1} \gamma_{K+1})\right)$. From \cite{Bartlett2017SpectrallynormalizedMB}, the margin $\rho(\cdot, \cdot)$ is 2-Lipschitz w.r.t network outputs, 
\[
|\rho(\hat{h}, \z) - \rho(h, \z)| \leq 2 \norm{\hat{h}(\x) - h(\x)}_{\infty}.
\]
Hence we can reduce the noise-resilience over the margin to the event that variation in networks outputs is bounded, 
\begin{align*}
&\prob_{\hat{h} \sim  \mathcal{N}\left(h, \bm{\sigma}^2_{\mathrm{sparse}}\right)
}
\big[ 
 |\rho(\hat{h},\z) - \rho(h, \vc{z})| > 2\zeta_{K+1} \gamma_{K+1} 
\big] \\
& \qquad \qquad \qquad \leq 
\prob_{\hat{h} \sim  \mathcal{N}\left(h, \bm{\sigma}^2_{\mathrm{sparse}}\right)
} 
\big[ 
 \norm{\hat{h}(\x) - h(\x)}_{\infty} > \zeta_{K+1} \gamma_{K+1} 
\big] 
\\
\therefore
&\prob_{\hat{h} \sim  \mathcal{N}\left(h, \bm{\sigma}^2_{\mathrm{sparse}}\right)
} 
\big[ 
 \norm{\hat{h}(\x) - h(\x)}_{\infty} > \zeta_{K+1} \gamma_{K+1} 
\big] \leq \frac{1}{\sqrt{m}}\\
&\qquad \qquad \qquad \implies h \text{ is noise-resilient w.r.t } \rho \text{ at } \z. 
\end{align*}
Therefore $h$ is not noise-resilient w.r.t  $\rho $ at $\z =(\x_0, y)$ implies 
\[
\prob_{\hat{h} \sim  \mathcal{N}\left(h, \bm{\sigma}^2_{\mathrm{sparse}}\right)
} 
\big[ 
 \norm{\hat{h}(\x) - h(\x)}_{\infty} >\zeta_{K+1} \gamma_{K+1} 
\big] > \frac{1}{\sqrt{m}}
\]
Thus the probability (over inputs) that a predictor is not noise-resilient can be bounded using 
the event that the change in network output is large, 
\begin{align}\label{eq:bound-mu}
\nonumber \mu_\cD \left(h, (\rho , \zeta_{K+1} \gamma_{K+1}\right) 
& =  
\prob_{\z \sim \cD} 
\bigg[
\prob_{\hat{h} \sim  \mathcal{N}\left(h, \bm{\sigma}^2_{\mathrm{sparse}}\right)
} 
\big[ 
|\rho(\hat{h},\z) - \rho(h, \vc{z})| > 2\zeta_{K+1} \gamma_{K+1} 
\big] 
> \frac{1}{\sqrt{m}}.
\bigg]\\
\nonumber & \leq \prob_{\z \sim \cD} 
\bigg[
\prob_{\hat{h} \sim  \mathcal{N}\left(h, \bm{\sigma}^2_{\mathrm{sparse}}\right)
} 
\big[ 
\norm{\hat{h}(\x) - h(\x)}_{\infty} > \zeta_{K+1} \gamma_{K+1} 
\big] 
> \frac{1}{\sqrt{m}}.
\bigg]\\
& \leq \prob_{\z \sim \cD} 
\bigg[
\prob_{\hat{h} \sim  \mathcal{N}\left(h, \bm{\sigma}^2_{\mathrm{sparse}}\right)
} 
\big[ 
\norm{\hat{h}(\x) - h(\x)}_{\infty} \leq \zeta_{K+1} \gamma_{K+1} 
\big] 
< 1- \frac{1}{\sqrt{m}}.
\bigg]
\end{align}
We now make two observations that together helps us upper bound the above probability,
\begin{enumerate}
    \item  From \Cref{proposition: reduced-sensitivity}, if the layer-wise sparse local radius at each layer $k$ is sufficiently large, 
\begin{align*}
\forall~1 \leq k \leq K,~~ 
r_k(h, \x_0) \geq \gamma_k 
\end{align*}
then for any 
perturbed network $\hat{h}$ is in $\mathcal{B}(h, \bm{\epsilon})$ and the distance between the network-output, 
\begin{align*}
\norm{\hat{h}(\x_0) - h(\x_0)}_{\infty} 
&\leq \zeta_{K+1} \cdot \gamma_{K+1} 
\end{align*}
\item The choice of variance in the theorem statement, $\bm{\sigma}^2_{\mathrm{sparse}} = \bm{\sigma^2}(\frac{1}{\sqrt{m}})$, the variance described in  
\Cref{lemma: random-perturb-network} for $\delta = \frac{1}{\sqrt{m}}$. Thus by \Cref{lemma: random-perturb-network}, 
with probability at least $(1-\frac{1}{\sqrt{m}})$ a randomly perturbed network $\hat{h} \sim \mathcal{N}\left(h, \bm{\sigma}^2\right)$ is within the neighbourhood $\mathcal{B}(h, \bm{\epsilon})$. 
\end{enumerate}
We can combine the above two observations to infer that at any input $\vc{z}=(\x,y) \sim \cD$, 
\begin{align*}
   &  \forall~1 \leq k \leq K,~~ 
r_k(h, \x_0) \geq \gamma_k \\ 
 & \qquad \implies   
\prob_{\hat{h} \sim  \mathcal{N}\left(h, \bm{\sigma}^2_{\mathrm{sparse}}\right)
} 
\big[ 
\norm{\hat{h}(\x) - h(\x)}_{\infty} \leq \zeta_{K+1} \gamma_{K+1} 
\big] 
\geq 1- \frac{1}{\sqrt{m}}.
\end{align*}
Thus to ensure that the probability that $\norm{\hat{h}(\x)-h(\x)}_{\infty} \leq \zeta_{K+1}\gamma_{K+1}$ is less than 
$1-\frac{1}{\sqrt{m}}$, one necessarily needs that the sparse local radius is insufficient at some layer $k$, 
\begin{align*}
&\prob_{\hat{h} \sim  \mathcal{N}\left(h, \bm{\sigma}^2_{\mathrm{sparse}}\right)
} 
\big[ 
\norm{\hat{h}(\x) - h(\x)}_{\infty} \leq \zeta_{K+1} \gamma_{K+1} 
\bigg] 
< 1- \frac{1}{\sqrt{m}} \\
&\qquad \qquad \implies  
 \exists~1 \leq k \leq K,~~ 
r_k(h, \x_0) < \gamma_k
\end{align*}
Plugging the above logic into Eq. \eqref{eq:bound-mu} we get a condition on the sparse local radius,
\begin{align*}
\mu_\cD \left(h, (\rho , 4\zeta_{K+1} \gamma_{K+1}\right) 
& =  
\prob_{\z \sim \cD} 
\bigg[
\prob_{\hat{h} \sim  \mathcal{N}\left(h, \bm{\sigma}^2_{\mathrm{sparse}}\right)
} 
\big[ 
|\rho(\hat{h},\z) - \rho(h, \vc{z})| > 2\zeta_{K+1} \gamma_{K+1} 
\big] 
> \frac{1}{\sqrt{m}}.
\bigg]\\
& \leq \prob_{\z \sim \cD} 
\bigg[
\prob_{\hat{h} \sim  \mathcal{N}\left(h, \bm{\sigma}^2_{\mathrm{sparse}}\right)
} 
\big[ 
\norm{\hat{h}(\x) - h(\x)}_{\infty} > \zeta_{K+1} \gamma_{K+1}
\big] 
> \frac{1}{\sqrt{m}}.
\bigg] \\ 
& \leq 
\prob_{\z \sim \cD} \left(\exists\; 1\leq k \leq K : \; r_k(h, \x_0) < \gamma_k 
 \right)
\end{align*}
Similarly we can reduce the noise-resilience condition on training sample, 
\begin{align*}
\mu_{\samp_T} \left(h, (\rho , 4 \zeta_{K+1} \gamma_{K+1}\right) 
& =  
\prob_{\z \sim \mathfrak{U}(\samp_T)} 
\bigg[
\prob_{\hat{h} \sim  \mathcal{N}\left(h, \bm{\sigma}^2_{\mathrm{sparse}}\right)
} 
\big[ 
|\rho(\hat{h},\z) - \rho(h, \vc{z})| > 2\zeta_{K+1} \gamma_{K+1}
\big] 
> \frac{1}{\sqrt{m}}.
\bigg]\\
& \leq 
\frac{1}{m} \sum_{
\z^{(i)} \in \samp_T
}
\mathbf{1}\left\{\exists \; 1\leq k \leq K: \; r_k(h, \x_0) < \gamma_k 
 \right\}
\end{align*}
To summarize we have the following generalization bound that holds with probability at least $(1-\frac{\delta}{K+1})$, 
\begin{align*}
&\prob_{\z \sim \cD} \bigg[ \rho(h, \z) < 0 \bigg] 
\\ \nonumber & \qquad  
\leq \frac{1}{m} \sum_{\z^{(i)} \in \samp_T} \mathbf{1}\bigg[ \rho(h, \z^{(i)}) < 4 \zeta_{K+1} \gamma_{K+1} \bigg] + \frac{2}{\sqrt{m}-1} 
\\
&\qquad   + 
\frac{1}{m} \sum_{
\z^{(i)} \in \samp_T
}
\mathbf{1}\left\{\exists \; 1\leq k \leq K: \; r_k(h, \x_0) < \gamma_k 
 \right\}
 \\
 &\qquad  +
 \prob_{\z \sim \cD} \left(\exists \; 1\leq k \leq K: \; r_k(h, \x_0) < \gamma_k 
 \right)\\
\nonumber 
&\qquad   + \sqrt{\frac{4\mathrm{KL}\left(  
  \mathcal{N}\left(h, \bm{\sigma}^2_{\mathrm{sparse}}\right)
 \;|| \; 
 \mathcal{P}
 \right) + 2\log\left(\frac{(K+1)m}{\delta}\right) 
}{m-1}} . 
\end{align*}
We still need to bound the probability that the sparse local radii aren't sufficiently large, 
\[
 \prob_{\z \sim \cD} \left(\exists \; 1\leq k \leq K: \; r_k(h, \x_0) < 
 \gamma_k 
 \right).
\]
Consider the set of properties $\{r_k(h, \x_0) - \gamma_k\}_{k=1}^{K}$
and margin thresholds $\{2\gamma_k\}_{k=1}^K$. 
\Cref{lemma: derandomization} shows that with probability at  least $(1-\frac{\delta}{(K+1)})$ over the choice of i.i.d training sample $\samp_T$ of size $m$, 
the generalization error is bounded by
\begin{align}
& \prob_{\z \sim \cD} \left(\exists \; 1\leq k \leq K: \; r_k(h, \z) < \gamma_k \right). \\
\nonumber 
& \leq \frac{1}{m} \sum_{\z^{(i)} \in \samp_T} \mathbf{1}\left\{
 \exists \; 1\leq k \leq K: \; 
 r_k(h, \z^{(i)}) \leq 3\gamma^k
 \right\} + \frac{2}{\sqrt{m}-1} \\
 &
+ \mu_{\samp_T} \left(h, \left\{
r_k-\gamma_k, 2\gamma^k
\right\}_{k=1}^K\right)
+ \mu_{\cD} \left(h, \left\{
r_k-\gamma_k, 2\gamma^k
\right\}_{k=1}^K\right)\\
\nonumber 
& + \sqrt{\frac{4\cdot \sum_{k=1}^K \mathrm{KL}\left(  
  \mathcal{N}\left(\W_k, \sigma^2_k\right)
 \;|| \;
 \mathcal{P}_k
 \right) + \log(\frac{2m(K+1)}{\delta}) 
}{m-1}} . 
\end{align}
To bound $\mu_{\cD} \left(h, \left\{
r_k-\gamma_k, 2\gamma^k
\right\}_{k=1}^K\right)$, we can instantiate a recursive procedure. 
By the choice of variance definition and \Cref{lemma: slr}, we note that at any input $\z$, for all $2\leq k \leq K$, 
\begin{align*} 
&\forall \; 1\leq n \leq k-1: \; r_n(h,\z) \geq \gamma_k\\
&\implies \text{ w. p. at least } (1-\frac{1}{\sqrt{m}}),\;
 \forall\; 2 \leq n \leq k,~~
 \left\lvert r_n(\hat{h},\z) - r_{n}(h, \z)\right\rvert \leq \gamma_n\\ 
 &\implies h \text{ is noise-resilient at } z \text{ w.r.t  properties }\{r_n - \gamma_n\}_{n=1}^{k} \text{at thresholds} \{2\gamma_n\}_{n=1}^{k}.
\end{align*}
In the above, we have also used the fact that the sparse local radius in the first layer is always noise-resilient at the specified $\gamma_1$ and choice of variance.
Hence, we have the following inequality for all $2\leq k \leq K$,
\begin{align*}
\mu_{\cD} \left(h, \left\{
r_n - \gamma_n, 2\gamma_n
\right\}_{n=1}^{k}\right) 
&\leq \prob_{\z \sim \cD} \left(\exists \; 1\leq n \leq k-1: \; r_{n}(h, \z) < \gamma_n\right) \\
\mu_{\samp_T} \left(h, \left\{
r_n - \gamma_n, 2\gamma_n
\right\}_{n=1}^{k}\right) 
&\leq \frac{1}{m} \sum_{\z^{(i)} \in \samp_T}\mathbf{1}\left(\exists \; 1\leq n \leq k-1: \; r_{n}(h, \z^{(i)}) < \gamma_n\right)
\end{align*}

We can now use this recursively bound the probabilty that the sparse local radii aren't sufficiently large, starting from from $K$,
\begin{align*}
& \prob_{\z \sim \cD} \left(\exists \; 1\leq n \leq k: \; r_n(h, \z) < \gamma_n \right). \\
& \leq \frac{1}{m} \sum_{\z^{(i)} \in \samp_T} \mathbf{1}\left\{
 \exists \; 1\leq n \leq k: \; 
 r_n(h, \z^{(i)}) \leq 3\gamma_n
 \right\} + \frac{2}{\sqrt{m}-1} \\
 &
+ \frac{1}{m} \sum_{\z^{(i)} \in \samp_T}\mathbf{1}\left\{\exists \; 1\leq n \leq k-1: \; r_{n}(h, \z^{(i)}) < \gamma_n\right\}\\
&+  \prob_{\z \sim \cD} \left(\exists \; 1\leq n \leq k-1: \; \rho^{n}(h, \z) < \gamma_n\right) \\ 
&  \sqrt{\frac{4\cdot \sum_{n=1}^{k} \mathrm{KL}\left(  
  \mathcal{N}\left(\W_n, \sigma^2_n\right)
 \;|| \;
 \mathcal{P}
 \right) + \log(\frac{2m(K+1)}{\delta}) 
}{2(m-1)}} . 
\end{align*}
The conclusion follows by plugging in these bounds recursively. 
\end{proof}
To prove \Cref{theorem: sparsity-aware-single} we note that the variance is strictly lesser than the variance in \Cref{theorem: sparsity-aware-expanded} and that the loss terms have been collapsed into the worst-case over layers resulting in a worse generalization bound.

\section{Hyper-parameter search}
\label{app:hyper-search}
We search for good base hyper-parameters $(\vc{s}, \bm{\xi}, \bm{\eta}, \bm{\epsilon})$ as described in \Cref{sec:HyperParameterSearch}. We base our search on the stronger bound in \Cref{theorem: sparsity-aware-expanded} rather than the simplified result in \Cref{theorem: sparsity-aware-single}. 
For any choice of sensitivity vector $\bm{\epsilon}$ and sparse risk control $\alpha$, we choose the sparsity vector to ensure that, 
\[
\sum_{k\in [K]} \frac{1}{m} \sum_{z^{(i)} \in \samp_T} \underbrace{\mathbf{1}\left\{\;\exists\; 1\leq n\leq k\;:~~ r_n(h,\z^{(i)}) \leq 3\gamma_n \right\}}_{=: \ell_{\vc{s}}(h, \z^{(i)}) } \leq \alpha
\]
Doing so automatically controls the other relaxed sparse loss term in \Cref{theorem: sparsity-aware-expanded}, 
\[
\sum_{k\in [K]} \frac{1}{m} \sum_{z^{(i)} \in \samp_T} \mathbf{1}\left\{\;\exists\; 1\leq n\leq k\;:~~ r_n(h,\z^{(i)}) \leq \gamma_n \right\}\] 
For each $(\bm{\epsilon}, \alpha)$ and input $\x^{(i)}$, the sparsity vector vector $s^*(\x^{(i)}, \bm{\epsilon}) = \{s^{(i)}_1, \ldots, s^{(i)}_K\}$ is decided in layer-wise fashion. 
At each layer $k$, having fixed the sparsity levels $\{s^{(i)}_1, \ldots, s^{(i)}_{k-1}\}$\footnote{We fix $s^{(i)}_0 = s^{(i)}_{K+1}=0$ for all $i$.}, one can fix the next sparsity level as, 
\begin{align*}
s^{(i)}_k &:= \max_{s \in [d_k]}~ s \quad \text{such that } r_k(h,\z^{(i)}) > 3\gamma_n. \\
\Leftrightarrow 
s^{(i)}_k &:= \max_{s \in [d_k]}~ s \quad \text{such that } \act{
\mathrm{sort}\left(-\left[ \frac{  \Wvec{k}{j} \cdot \vc{x}^{(i)}_{k-1} }{
\xi_k \cdot \zeta_{k-1}
}  \right]_{j=1}^{d_k},\; s\right) 
} > 3\left(-1+\prod_{n=1}^k (1+\epsilon_k)\right)
\end{align*}
where for each feasible $s$, $\xi_k$ is the closest bound to the relevant sparse norm $\norm{\W_k}_{(d_k-1, s^{(i)}_{k-1})}$ and $\zeta_{k-1}$ is the bound on the scale of the layer input $\x^{(i)}_{k-1}$ dependent on the previously fixed sparsity levels, i.e.
$\zeta_{k-1} := \mathsf{M}_{\cX_0}\prod_{n=1}^{k-1} \xi_n \sqrt{1+\eta_n}$ where for each $1\leq n\leq k-1$, $\norm{\W_{n}}_{(d_{n}-1, s^{(i)}_{n-1})} \leq \xi_n$ and $\mu_{(s^{(i)}_{n}, s^{(i)}_{n-1})}(\W_n) \leq \eta_n$.  
Under the choice of the $s^{*}(\vc{x}^{(i)}, \bm{\epsilon})$, the sparse loss $\ell_{s^{*}(\vc{x}^{(i)}, \bm{\epsilon})}(h, \z^{(i)}) = 0$. 
Further under the choice of the sample-wide minimum sparsity vector $\bar{s}(\bm{\epsilon})$, based on  $s^{*}(\vc{x}^{(i)}, \bm{\epsilon})$, i.e.  $\bar{s}_k := \min_{i \in [m]} s^{(i)}_k$, the average sparse loss $\ell_{\bm{\vc{s}}}$ is zero. 

To control the sparse loss, it is sufficent to consider the quantiles of the distribution of $s^{*}(\x, \bm{\epsilon})$ by the training samples. 
We denote by $\hat{s}(\bm{\epsilon}, \alpha)$ with layer-wise sparsity levels  
\[
\hat{s}_k:= \texttt{quantile}\left( \{s^{*}(\x^{(i)}, \bm{\epsilon})\}_{i=1}^m  , \frac{2\alpha}{K(K-1)}\right).
\]
Under such a choice, 
\[
\frac{1}{m} \sum_{z^{(i)} \in \samp_T} ~~ \mathbf{1}\left\{r_k(h,\z^{(i)}) \leq 3\gamma_k \right\} \leq \frac{2\alpha}{K(K-1)}
\]
Hence we can see that, 
\begin{align*}
&\sum_{k\in [K]} \frac{1}{m} \sum_{z^{(i)} \in \samp_T} \mathbf{1}\left\{\;\exists\; 1\leq n\leq k\;:~~ r_n(h,\z^{(i)}) \leq 3\gamma_n \right\}\\
&\leq 
\frac{1}{m} \sum_{z^{(i)} \in \samp_T} 
\sum_{k \in [K]} 
\sum_{n \in [k-1]}~~ \mathbf{1}\left\{r_n(h,\z^{(i)}) \leq 3\gamma_n \right\} \\
&\leq \frac{1}{m} \sum_{z^{(i)} \in \samp_T} 
\sum_{k \in [K]} 
\sum_{n \in [k-1]}~~ \frac{2\alpha}{K(K-1)} \\
&\leq \frac{1}{m} \sum_{z^{(i)} \in \samp_T}  \alpha \leq \alpha.
\end{align*}
Thus we have seen how to control the average sparse loss $\ell_{\vc{s}}$ for a fixed sensitivity vector $\bm{\epsilon}$ and control threshold $\alpha$. 
As a final simplification, we note that by the nature of the sensitivity analysis, it is more important that the sparse local radius at lower layers is sufficiently large as compared to later layers (for eg, the last layer only factors into one of the loss terms). Hence in our experiments, we let $\epsilon_k = \frac{\bar{\epsilon}}{K+1-k}$ at all layers for some fixed $\bar{\epsilon} \in [0,1]$. 
We now search for the best-in-grid generalization bound in the search space $[0,1]\times [0,1]$ to find the find the best-in-grid choice of hyper-parameters $(\bar{\epsilon}, \alpha)$. 

\section{PAC-Bayes Tools}
In this section, we discuss results from PAC-Bayesian analysis old and new. For the sake of completeness, we first state the classical PAC-Bayes generalization theorem from \cite{McAllester1998SomePT, ShalevShwartz2014UnderstandingML}. Unlike Rademacher analysis, PAC-Bayes provides generalization bounds on a stochastic network. We then quote a useful de-randomization argument from \cite{nagarajan2018deterministic} that provides generalization bounds on the mean network. 

\subsection{Standard PAC-Bayes theorems}
\begin{theorem}(\cite{McAllester1998SomePT, ShalevShwartz2014UnderstandingML})\label{app: theorem: PAC-Bayes}
Let $\cD$ be an arbitrary distribution over data $\cZ = \cX \times \cY$. Let $\cH$ be a hypothesis class and let $\ell : \cH \times \cZ \rightarrow [0,1]$ be a loss function. Let $\mathcal{P}$ be a prior distribution over $\mathcal{H}$ and let $\delta \in (0,1)$. 
Let $\samp = \{\vc{z}_1, \ldots, \vc{z}_m\}$ be a set of i.i.d training samples from $\cD$. 
Then, with probability of at least $(1-\delta)$ over the choice of training sample $\samp$, for all distributions $\mathcal{Q}$ over $\cH$ (even such that depend on $\samp$) we have 
\begin{align}
\label{eq: PAC-Bayes}
\expect_{\hat{h}\sim \mathcal{Q}}\; \expect_{(\x,\y) \sim \cD_\cZ} \big[ \ell(\hat{h}, (\x,\y)) \big] 
\leq 
\expect_{\hat{h}\sim \mathcal{Q}}\; \expect_{(\x,\y) \sim \mathrm{Unif}(\samp_T)} \big[ \ell(\hat{h}, (\x,\y)) \big] 
+ \sqrt{\frac{\mathrm{KL}\left(  \mathcal{Q} || \mathcal{P}\right) + \log(\frac{m}{\delta}) 
}{2(m-1)}}.
\end{align}
Note that \Cref{app: theorem: PAC-Bayes} bounds the generalization error of a stochastic predictor $\hat{h} \sim \mathcal{Q}$. 
\end{theorem}
\subsection{De-randomized PAC-Bayes theorems}\label{app: subsec: derandom}
Let $\cD$ be an arbitrary distribution over data $\cZ = \cX \times \cY$ and let $\samp = \{\vc{z}_1, \ldots, \vc{z}_m\}$ be a set of i.i.d training samples from $\cD$. 
Let $\cH$ be a hypothesis class and let $\mathcal{P}$ be a prior distribution over $\mathcal{H}$ and let $\delta \in (0,1)$. Consider a fixed predictor $h \in \cH$ (possibly trained on data).  
Let $\mathcal{Q}(h, \Sigma)$ be any distribution over $\cH$ (possibly data dependent) with mean $h$ and covariance $\Sigma$. 
Let $\rho_n(h, \z)$ for $1\leq n \leq N$ be certain properties of the predictor $h$ on data point $\z$ and let $\gamma_n >0$ be its associated margin.  
\begin{definition}(Noise-resilience, \cite{nagarajan2018deterministic})
A predictor $h$ is said to be \textit{noise-resilient} at a given data point $\z$  w.r.t properties $\rho_n$ if, 
\begin{equation}\label{eq: noise-res}
\prob_{\hat{h} \sim \mathcal{Q}(h, \Sigma)} 
\big[ 
\exists n : |\rho_n(\hat{h},\z) - \rho_n(h, \vc{z})| > \frac{\gamma_n}{2} 
\big] 
\leq \frac{1}{\sqrt{m}}.
\end{equation}
Let $\mu_\cD \left(h, \{ (\rho_n , \gamma_n) \}_{n=1}^N \right)$ denote the probability that $h$ is not noise-resilient at a randomly drawn $\z \sim \cD$, 
\[
\mu_\cD \left(h, \{ (\rho_n , \gamma_n) \}_{n=1}^N \right) := 
\prob_{\z \sim \cD} 
\bigg
[ \prob_{\hat{h} \sim \mathcal{Q}(h, \Sigma)} 
\big[ 
\exists n : |\rho_n(\hat{h},\z) - \rho_n(h, \vc{z})| > \frac{\gamma_n}{2} 
\big] 
> \frac{1}{\sqrt{m}}.
\bigg]
\]
Similarly let $\mu_{\samp} \left(h, \{ (\rho_n , \gamma_n) \}_{n=1}^N \right)$ denote the probability that $h$ is not noise-resilient at a randomly drawn training sample $\z \sim \mathfrak{U}(\samp)$, 
\[
\mu_{\samp_T} \left(h, \{ (\rho_n , \gamma_n) \}_{n=1}^N \right) := 
\prob_{\z \sim \mathfrak{U}(\samp_T)} 
\bigg
[ \prob_{\hat{h} \sim \mathcal{Q}(h, \Sigma)} 
\big[ 
\exists n : |\rho_n(\hat{h},\z) - \rho_n(h, \vc{z})| > \frac{\gamma_n}{2} 
\big] 
> \frac{1}{\sqrt{m}}.
\bigg]
\]
\end{definition} 
\begin{lemma}\cite[Theorem C.1]{nagarajan2018deterministic}\label[lemma]{lemma: derandomization}
For some fixed margin hyper-parameters $\{\gamma_n\}_{n=1}^N$, with probability at least $(1-\delta)$ over the draw of training sample $\samp$, for any predictor $h$ and any , we have, 
\begin{align}
\label{eq: De-random}
&\prob_{\z \sim \cD} \bigg[ \exists n : \rho_n(h, \z) < 0 \bigg] \\
\nonumber & \leq \frac{1}{m} \sum_{\z^{(i)} \in \samp_T} \mathbf{1}\bigg[ \exists n : \rho_n(h, \z^{(i)}) < \gamma_n \bigg]
+ \mu_{\samp_T} \left(h, \{ (\rho_n , \gamma_n) \}_{n=1}^N \right)
+ \mu_\cD \left(h, \{ (\rho_n , \gamma_n) \}_{n=1}^N \right)\\
\nonumber & + 2\sqrt{\frac{2\mathrm{KL}\left(  \mathcal{Q}(h,\Sigma) \;|| \;\mathcal{P}\right) + \log(\frac{2m}{\delta}) 
}{2(m-1)}} + \frac{2}{\sqrt{m}-1}. 
\end{align}
\end{lemma}

In the above lemma, the loss function $\ell(h, \z) := \mathbf{1}\left\{ \exists\; n, \;\rho(h,\z) < 0 \right\}$.
\Cref{lemma: derandomization} directly bounds the generalization of a predictor $h$ rather than a stochastic predictor $\hat{h} \sim \mathcal{Q}(h, \Sigma)$. 

\section{Sparse Norm, Reduced Babel Function and Gaussian Concentration}\label{app: sparse-norm}
\begin{lemma}\label[lemma]{app: lemma: monotone-sparse-norm}
For any matrix $\W \in \mathbb{R}^{d_2 \times d_1}$ and any two sparsity levels such that $(s_2, s_1) \preceq (\hat{s}_2, \hat{s}_1)$, 
$\norm{\W}_{(\hat{s}_2, \hat{s}_1)} \leq \norm{\W}_{(s_2, s_1)}$
\end{lemma}
\begin{proof}
Let $\hat{J}_1 \subseteq [d_1], |J_1| = d_1-\hat{s}_1$ and $\hat{J}_2 \subseteq [d_2], |\hat{J}_2| = d_2 - \hat{s}_2$ be two index sets such that $\norm{\W}_{(\hat{s}_2, \hat{s}_1)} = \norm{\mathcal{P}_{\hat{J}_2, \hat{J}_1}(\W)}_2$. 
Consider any two extended index sets $J_1 \subseteq [d_1], |J
_1| = d_1-s_1$ and $J_2 \subset [d_2], |J_2| = d_2-s_2$ such that $\hat{J}_1 \subseteq J_1$ and $\hat{J}_2 \subseteq J_2$. Then, 
\begin{align*}
\norm{\W}_{(\hat{s}_2, \hat{s}_1)} 
&= \norm{\mathcal{P}_{\hat{J}_2, \hat{J}_1}(\W)}_2 \\
&\leq \norm{\mathcal{P}_{J_2, J_1}(\W)}_2 \\
&\leq \max_{J_2 \subset [d_2], |J_2| = d_2-s_2} \max_{J_1 \subseteq [d_1], |J
_1|= d_1-s_1} \norm{\mathcal{P}_{J_2, J_1}(\W)}_2 \\
&=: \norm{\W}_{(s_2,s_1)}.
\end{align*}
\end{proof}
Recall the definition of reduced babel function from \ref{def: red-babel}, 
\[
\mu_{s_2,s_1}(\W) := 
\frac{1}{\norm{\W}^2_{(d_2-1, s_1)}}\;
\underset{\substack{J_2 \subset [d_2],\\ |J_2|=d_2-s_2}}{\max} \;
\max_{j \in J_2} 
\Bigg[
\sum_{\substack{i \in J_2,\\ i \neq j}}\;
\underset{\substack{J_1 \subseteq [d_1]\\ |J_1| = d_1 - s_1}}{\max}
\left\lvert\mathcal{P}_{J_1}(\vc{w}[i]) \mathcal{P}_{J_1}(\vc{w}[j])^T
\right\rvert 
\Bigg],
\]
To compute this we note, that maximum reduced inner product $\underset{\substack{J_1 \subseteq [d_1]\\ |J_1| = d_1 - s_1}}{\max}
\left\lvert\mathcal{P}_{J_1}(\vc{w}[i]) \mathcal{P}_{J_1}(\vc{w}[j])^T\right\rvert$ can be computed taking the sum of the top-k column indices in an element-wise product of rows $\vc{w}[i]$ and $\vc{w}[j]$. 
The full algorithm to compute the babel function is described in \Cref{alg: red-babel}. One can note that it has a computational complexity of $\mathcal{O}(d_2^2 d_1)$ for each $\mu_{s_2, s_1}(\W)$. Further optimizations that leverage PyTorch broadcasting are possible. The reduced babel function is useful as it provides a bound on the sparse norm. 
\begin{algorithm}[hbt!]
\caption{Computing the Reduced Babel function}\label{alg: red-babel}
\begin{algorithmic}
\State \textbf{Require} :  Weight matrix $W \in \mathbb{R}^{d_2 \times d_1}$, sparsity levels $s_2 \in d_2-1$ and $0 \leq s_1 \leq d_1-1$. 
\State \textbf{Ensure} :  The reduced babel function at specified sparsity, $\mu_{s_2, s_1}(\W)$.
\State \textbf{Initialize} : A vector of Gerschgorin disk radii, $\vc{r} = \mathbf{0}_{d_0}$
\State \textbf{Initialize} : A matrix of maximum reduced inner products, $\mt{A} = \mathbf{0}_{(d_2 \times d_2)}$. 
\State \textbf{Initialize}: A vector of top-k elementwise squares $\vc{t} = \mathbf{0}_{d_2}$
\State $\vc{t}[i] = \textsc{sum}(\textsc{Top-k}(\vc{w}[i] \circ \vc{w}[i], d_1-s_1))$
\State $\norm{\W}_{d_2-1,s_1} = \sqrt{\max_i \vc{t}[i]}$
\State Compute maximum reduced inner product for each $(i,j)$
\For{$1\leq i\leq d_2$}
    \For{$1\leq j \leq d_2,\; j\neq i$}
        \State positive = \textsc{Top-k}$(\vc{w}[i] \circ \vc{w}[j], d_1-s_1)$
        \State negative = \textsc{Top-k}$(\vc{w}[i] \circ \vc{w}[j], d_1-s_1)$
        \State $A[i, j] = \max\{\textsc{Sum}(\text{positive}),\;\textsc{Sum}(\text{negative})\}$.
    \EndFor
\EndFor

\State Compute gerschgorin radii.
\For{$1\leq i \leq d_2-1$}
    \State $r[i] = \textsc{SUM}( \textsc{Top-k}(\vc{a}[i], d_1-s_1   )  )$
\EndFor

\State \textbf{Return}: $\mu_{s_2, s_1}(\W) = \frac{\max_i \vc{r}[i]}{\norm{\W}_{d_2-1,s_1}}$.

\end{algorithmic}
\end{algorithm}

\begin{lemma}\cite[Lemma 3]{https://doi.org/10.48550/arxiv.2202.13216}
\label[lemma]{lemma: bound-submatrix-norm}
For any matrix $\W \in \mathbb{R}^{d_2 \times d_!}$, 
the sparse norm can be bounded as
\[
\norm{\W}_{(s_2,s_1)} \leq \norm{\W}_{(d_2-1,s_1)} \sqrt{1+\mu_{s_2, s_1}(\W)}.
\]
\end{lemma}
\begin{proof}
Despite the slight modifications to the reduced babel function definition and the novel sparse norm definition $\norm{\W}_{(d_2-1,s_1)}$, the proof follows an identical series of steps as in \cite{https://doi.org/10.48550/arxiv.2202.13216}.
\end{proof} 
While \Cref{lemma: bound-submatrix-norm} presents a useful deterministic bound for the sparse norm. We can also present a high probability bound on the sparse norm of a Gaussian matrix. To start off we present some well-known lemmas, 

\begin{lemma}(Concentration of norm of a sub-Gaussian sub-Matrix) \label[lemma]{lemma:submatrixnorm}
The operator norm of a sub-matrix indexed by sets $J_2 \subseteq [d_2]$ of size $(d_2-s_1)$ and $J_1\subseteq [d_1]$ of size $(d_1-s_1)$ is bounded in high probability, 
\[
\prob\left(\norm{\mathcal{P}_{J_2, J_1}(\mt{A})}_2 \geq \sigma(\sqrt{d_2-s_2}+\sqrt{d_1-s_1}+ t) \right) \leq e^{-\frac{t^2}{2}}, \quad \forall~t\geq 0.
\]
\end{lemma}
\begin{proof}
    This is a straightforward application of a classical result on the concentration of norm of Gaussian Matrix \cite[Theorem 6.1]{wainwright_2019} instantiated for the submatrix $\mathcal{P}_{J_2, J_1}(\mt{A})$. 
\end{proof}
\begin{lemma}(Concentration of sparse norm)\label[lemma]{app: lemma: sparse-norm}
For sparsity level $0 \leq s_2 \leq d_2-1$ and $0 \leq s_1 \leq d_1-1$, 
the operator norm of any sub-matrix indexed of size $(d_2-s_2) \times (d_1-s_1)$ is bounded in high probability, 
\[
\prob\left( \norm{\mt{A}}_{(s_2,s_1)} \geq \sigma(\sqrt{d_2 - s_2}+\sqrt{d_1-s_1}+ t) \right) \leq \binom{d_2}{s_2}\binom{d_1}{s_1}e^{-\frac{t^2}{2}}, \quad \forall~t\geq 0.
\]
Hence w.p. at least $(1-\delta)$, 
\[
\norm{\mt{A}}_{(s_2,s_1)}
 \leq \sigma\left(\sqrt{d_2-s_2}+\sqrt{d_1-s_1}) + \sqrt{2\log\binom{d_2}{s_2} + 2\log\binom{d_1}{s_1} + 2\log\left(\frac{1}{\delta}\right)}\right).
\]
\end{lemma}
\begin{proof}
Recall that $\max_{\substack{|J_2|=d_2-s_2, \\ |J_1|=d_1-s_1   }} \norm{\mathcal{P}_{J_2, J_1}(\mt{A})}_2$. 
For each $S_2$, $S_1$, by \Cref{lemma:submatrixnorm}, we have that, 
\[
\prob\left(\norm{\mathcal{P}_{J_2, J_1}(\mt{A})}_2 \geq \sigma(\sqrt{d_2-s_2}+\sqrt{d_1-s_1}+ t) \right) \leq e^{-\frac{t^2}{2}}, \quad \forall~t\geq 0.
\]
Thus, 
\begin{align*}
    &\prob\left( \max_{\substack{|J_2|=d_2-s_2 \\ |J_1|=d_1-s_1   }} \norm{\mathcal{P}_{J_2, J_1}(\mt{A})}_2  \geq \sigma(\sqrt{d_2-s_2}+\sqrt{d_1-s_1}+ t)\right) \\
    &\leq 
    \prob\left( \exists J_2,J_1,~   \norm{\mathcal{P}_{J_2, J_1}(\mt{A})}_2 \geq \sigma(\sqrt{d_2-s_2}+\sqrt{d_1-s_1}+ t)\right) \\
    &\leq \sum_{\substack{|J_2|=d_2-s_2 \\ |J_1|=d_1-s_1   }} 
    \prob\left(| \norm{\mathcal{P}_{J_2, J_1}(\mt{A})}_2| \geq \sigma(\sqrt{d_2-s_2}+\sqrt{d_1-s_1}+ t)\right)\\
    &\leq \binom{d_2}{d_2-s_2}\binom{d_1}{d_1-s_1}e^{-\frac{t^2}{2}} = \binom{d_2}{s_2}\binom{d_1}{s_1}e^{-\frac{t^2}{2}}.
\end{align*}
Hence w.p. at least $(1-\delta)$, 
\begin{align*}
&\max_{\substack{|J_2|=d_2-s_2 \\ |J_1|=d_1-s_1   }} \norm{\mathcal{P}_{J_2, J_1}(\mt{A})}_2 
\leq \sigma\left(\sqrt{d_2-s_2}+\sqrt{d_1-s_1}) + \sqrt{2\log\binom{d_2}{s_2} + 2\log\binom{d_1}{s_1} + 2\log\left(\frac{1}{\delta}\right)}\right).
\end{align*}

\end{proof}

\section{Additional Experiments}\label{app:exp}
\begin{figure}[h]
    \centering
    \includegraphics[width=0.7\textwidth]{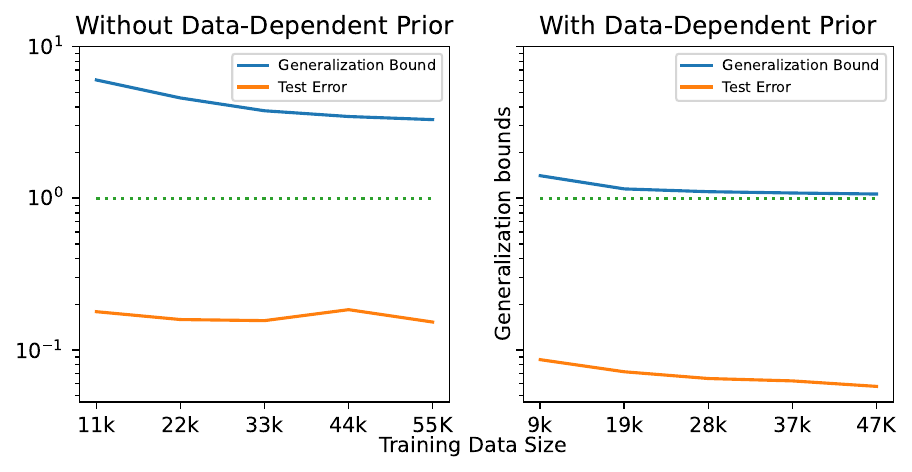}
    \caption{Generalization error of a 2-hidden layer model of width [100, 100] trained on MNIST}
    \label{fig:gen100_100}
\end{figure}

\begin{figure}[H]
    \centering
    \includegraphics[width=0.7\textwidth]{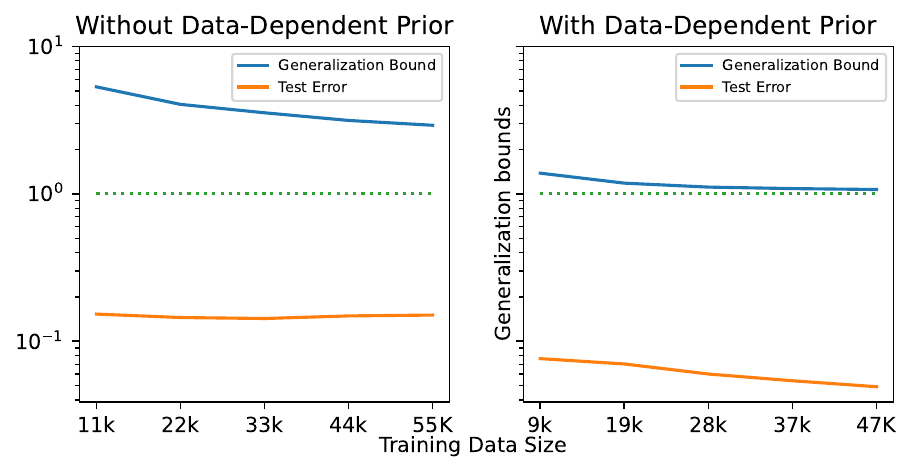}
    \caption{Generalization error of a 2-hidden layer model of width [500, 500] trained on MNIST}
    \label{fig:gen500_500}
\end{figure}

\begin{figure}
    \centering
\begin{minipage}{0.49\textwidth}
\includegraphics[width=\textwidth]{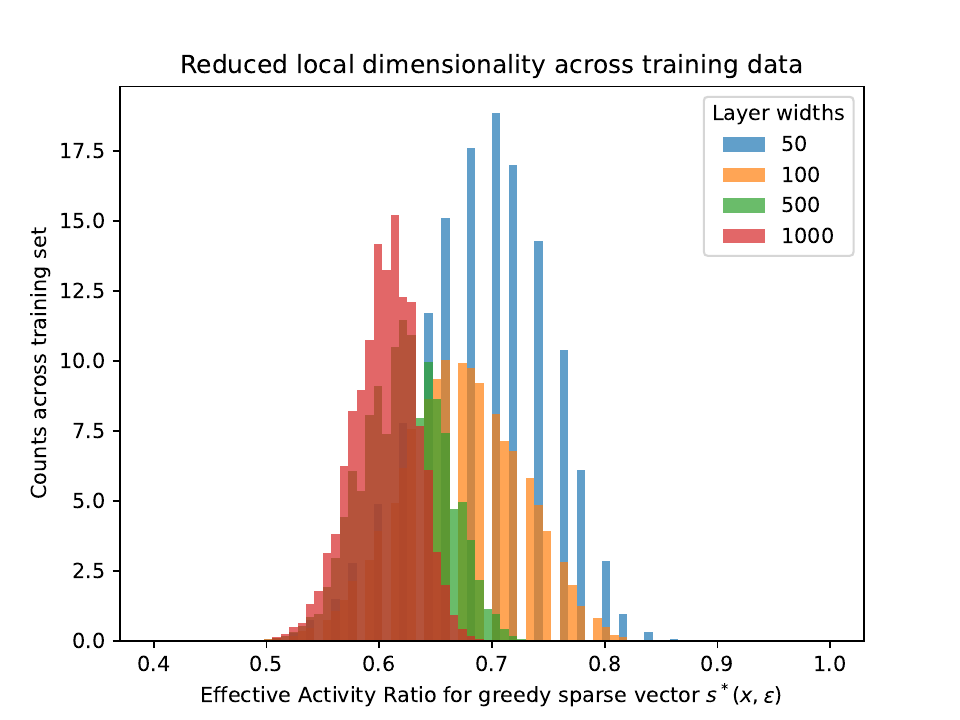}
\subcaption{Histogram of Effective Activity Ratio at $\epsilon = 10^{-4}$.}
\label{fig:HEA2}
\end{minipage}%
\begin{minipage}{0.49\textwidth}
\includegraphics[width=\textwidth]{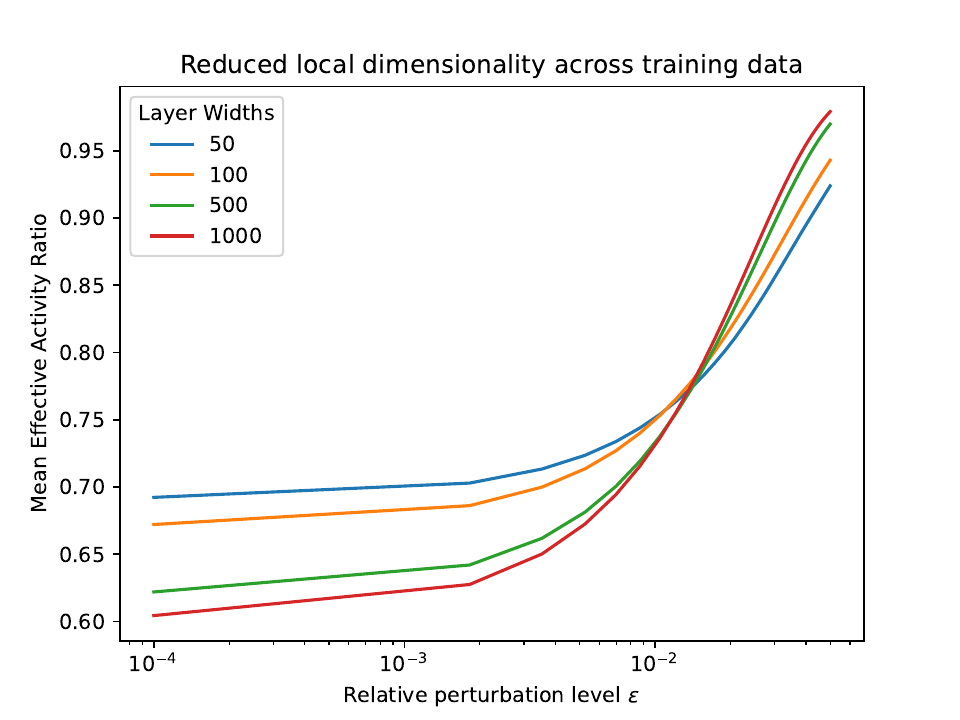}
\subcaption{Average Effective Activity Ratio.}
\label{fig:AEA2}
\end{minipage}
\caption{Effective Activity ratio $\kappa(\x, \epsilon)$ based on greedy sparsity vector $s^{*}(\x,\epsilon)$ for 2-layer networks (smaller implies sparser stable activations).}
\label{fig:2layer}
\end{figure}%

\end{document}